\def\year{2021}\relax
\documentclass[letterpaper]{article} 
\usepackage{aaai21}  
\usepackage{times}  
\usepackage{helvet} 
\usepackage{courier}  
\usepackage[hyphens]{url}  
\usepackage{graphicx} 
\usepackage{natbib}  
\usepackage{caption} 
\usepackage{amsmath}
\usepackage{amsfonts}
\usepackage{amssymb}
\usepackage{amsthm}
\usepackage{xcolor}
\usepackage{mathrsfs}
\usepackage{mathtools}
\usepackage{comment}
\usepackage{subfigure}
\usepackage{widetext}
\usepackage{gensymb}
\usepackage{amsmath}
\usepackage{amsmath,amsfonts,amssymb}
\usepackage{graphicx}
\usepackage{bm}
\usepackage{algorithm,algorithmicx}
\usepackage{threeparttable}
\usepackage{multirow}
\usepackage{booktabs}
\usepackage{tablefootnote}
\usepackage{array}
\usepackage{caption}
\usepackage{algpseudocode}

\DeclareMathOperator*{\argmin}{arg\,min}

\newtheorem{thm}{Theorem}
\newtheorem{property}{Property}
\newtheorem{observation}{Observation}
\newtheorem{lem}[thm]{Lemma}

\newtheorem{defn}{Definition}

\newtheorem*{rem}{Remark}

\title{Measuring Dependence with Matrix-based Entropy Functional}

\author {
    Shujian Yu\textsuperscript{\rm 1}, Francesco Alesiani\textsuperscript{\rm 1}, Xi Yu\textsuperscript{\rm 2}, Robert Jenssen\textsuperscript{\rm 3}, Jos\'{e} C. Pr\'{i}ncipe\textsuperscript{\rm 2}\\
}
\affiliations {
    \textsuperscript{\rm 1} NEC Labs Europe \\
    \textsuperscript{\rm 2} University of Florida \\
    \textsuperscript{\rm 3} UiT - The Arctic University of Norway \\
    yusj9011@gmail.com or Shujian.Yu@neclab.eu
}

\begin{document}

\maketitle

\begin{abstract}
Measuring the dependence of data plays a central role in statistics and machine learning. In this work, we summarize and generalize the main idea of existing information-theoretic dependence measures into a higher-level perspective by the Shearer's inequality. Based on our generalization, we then propose two measures, namely the matrix-based normalized total correlation ($T_\alpha^*$) and the matrix-based normalized dual total correlation ($D_\alpha^*$), to quantify the dependence of multiple variables in arbitrary dimensional space, without explicit estimation of the underlying data distributions. We show that our measures are differentiable and statistically more powerful than prevalent ones. We also show the impact of our measures in four different machine learning problems, namely the gene regulatory network inference, the robust machine learning under covariate shift and non-Gaussian noises, the subspace outlier detection, and the understanding of the learning dynamics of convolutional neural networks (CNNs), to demonstrate their utilities, advantages, as well as implications to those problems. Code of our dependence measure is available at: \url{https://bit.ly/AAAI-dependence}.
\end{abstract}





\section{Introduction}

Measuring the strength of dependence between random variables plays a central role in statistics and machine learning. For the linear dependence case, measures such as the Pearson's $\rho$, the Spearman's rank and the Kendall's $\tau$ are computationally efficient and have been widely used. For the more general case where the two variables share a nonlinear relationship, one of the most well-known dependence measures is the mutual information and its modifications such as the maximal information coefficient~\cite{reshef2011detecting}.



However, real-world data often contains three or more variables which can exhibit higher-order dependencies. If bivariate based measures are used to identify multivariate dependence, wrong conclusions may drawn. For example, in the XOR gate, we have $\mathbf{y}=\mathbf{x}^1\oplus \mathbf{x}^2$ with $\mathbf{x}^1$, $\mathbf{x}^2$ being binary random processes with equal probability. Although $\mathbf{x}^1$, $\mathbf{x}^2$ individually are independent to $\mathbf{y}$, the full dependence is synergistically contained in the union of $\{\mathbf{x}^1,\mathbf{x}^2\}$ and $\mathbf{y}$.


On the other hand, in various practical applications, the observational data or variables of interest lie on a high-dimensional space. Thus, it is desirable to extend the theory of scalar variable dependence to an arbitrary dimension.



Despite that tremendous efforts have been made based on the seven postulates (on measure of dependence on pair of variables) proposed by Alfr\'{e}d R\'{e}nyi~\cite{renyi1959measures}, the problem of measuring dependence (especially in a nonparametric manner) still remains challenging and unsatisfactory~\cite{fernandes2010mutual}. This is not hard to understand. Note that, most of the existing measures are defined as some functions of a density. Thus, a prerequisite for them is to estimate the underlying data distributions, a notoriously difficult problem in high-dimensional space.

Moreover, current measures primarily focus on two special scenarios: 1) the dependence associated with each dimension of a random vector (e.g., the multivariate maximal correlation (MAC)~\cite{nguyen2014multivariate}); and 2) the dependence between two random vectors (e.g., the Hilbert Schmidt Independence Criterion (HSIC)~\cite{gretton2005measuring}). The former is called multivariate correlation analysis in machine learning, and the latter is commonly referred to as random vector association in statistics.



Our main contributions are multi-fold:
\begin{itemize}
\item We provide a unified view of existing information-theoretic dependence measures and illustrate their inner connections. We also generalize the main idea of these measures into a higher-level perspective by the Shearer's inequality~\cite{chung1986some}.
\item Motivated by our generalization, we suggest two measures, namely the matrix-based normalized total correlation ($T_\alpha^*$) and the matrix-based normalized dual total correlation ($D_\alpha^*$), to quantify the dependence of data by making use of the recently proposed matrix-based R\'{e}nyi's $\alpha$-entropy functional estimator~\cite{giraldo2014measures,yu2019multivariate}.
\item We show that $T_\alpha^*$ and $D_\alpha^*$ enjoy several appealing properties. First, they are not constrained by the number of variables and variable dimension. Second, they are statistically more powerful than most of the prevalent measures. Moreover, they are differentiable, which make them suitable to be used as loss functions to train neural networks.
\item We show that our measures offer a remarkable performance gain to benchmark methods in applications like gene regulatory network (GRN) inference and subspace outlier detection. They also provide insights to challenging topics like the understanding of the dynamics of learning of Convolutional Neural Networks (CNNs).
\item Motivated by~\cite{greenfeld2020robust} that training a neural network by encouraging the distribution of the prediction residuals $e$ is statistically independent of the distribution of the input $\mathbf{x}$, we show that our measure (as a loss) improves robust machine learning against the shift of the input distribution ($a.k.a.$, the covariate shift~\cite{sugiyama2008direct}) and non-Gaussian noises. We also establish the connection between our loss to the minimum error entropy (MEE) criterion~\cite{erdogmus2002error}, a learning principle that has been extensively investigated in signal processing and process control.
\end{itemize}

\section{Background Knowledge}
\subsection{Problem Formulation}\label{sec:formulation}
We consider the problem of estimating the total amount of dependence of the $d_m$-dimensional components of the random variable $\mathbf{y}=[\mathbf{y}^1;\mathbf{y}^2;\cdots;\mathbf{y}^L]\in \mathbb{R}^d$, in which the $m$-th component $\mathbf{y}^m\in \mathbb{R}^{d_m}$ and $d=\sum_{m=1}^{L}d_m$. The estimation is based purely on $N$ \emph{i.i.d.} samples from $\mathbf{y}$, i.e., $\{\mathbf{y}_i\}_{i=1}^N$. Usually, we expect the derived statistic to be strictly between $0$ and $1$ for improved interpretability~\cite{wang2017unbiased}.

Obviously, when $L=2$, we are dealing with random vector association between $\mathbf{y}^1\in \mathbb{R}^{d_1}$ and $\mathbf{y}^2\in \mathbb{R}^{d_2}$. Notable measures in this category include the HSIC, the Randomized Dependence Coefficient (RDC)~\cite{lopez2013randomized},
the Cauchy-Schwarz quadratic mutual information (QMI\_CS)~\cite{principe2000learning} and the recently developed mutual information neural estimator (MINE)~\cite{belghazi2018mutual}. On the other hand, in case of $d_i=1$ $(i=1,2,\cdots,L)$, the problem reduces to the multivariate correlation analysis on each dimension of $\mathbf{y}$. Examples in this category are the multivariate Spearman's $\rho$~\cite{schmid2007multivariate} and the MAC.




Different from the above mentioned measures, we seek a general measure that is applicable to multiple variables in an arbitrary dimensional space (i.e., without constrains on $L$ and $d_i$). But, at the same time, we also hope that our measure is interpretable and statistically more powerful than existing counterparts in quantifying either random vector association or multivariate correlation.

\subsection{A Unified View of Information-Theoretic Measures}\label{sec:generalization}
From an information-theoretic perspective, a dependence measure $M$ that quantifies how much a random vector $\mathbf{y}=\{\mathbf{y}^1;\mathbf{y}^2;\cdots;\mathbf{y}^L\}\in \mathbb{R}^d$ deviates from statistical independence in each component can take the form of:
\begin{equation}\label{eq:diff}
M(\mathbf{y})=\mathsf{diff}\left(\mathsf{Pr}\left(\mathbf{y}^1,\mathbf{y}^2,\cdots,\mathbf{y}^L\right):\prod_{i=1}^{L}\mathsf{Pr}\left(\mathbf{y}^i\right)\right),
\end{equation}
where $\mathsf{diff}$ refers to a measure of difference such as divergence or distance.

If one instantiates $\mathsf{diff}\left(\right)$ with Kullback¨CLeibler (KL) divergence, Eq.~(\ref{eq:diff}) reduces to the renowned Total Correlation~\cite{watanabe1960information}:
\begin{eqnarray}\label{eq:TC_org}
T(\mathbf{y})&=&D_{\text{KL}}\left(\mathsf{Pr}\left(\mathbf{y}^1,\mathbf{y}^2,\cdots,\mathbf{y}^L\right)||\prod_{i=1}^{L}\mathsf{Pr}\left(\mathbf{y}^i\right)\right),\nonumber \\
&=&\left[\sum_{i=1}^{L}{H(\mathbf{y}^i)}\right]-H(\mathbf{y}^1,\mathbf{y}^2,\cdots,\mathbf{y}^L),
\end{eqnarray}
where $H$ denotes entropy or joint entropy.

Most of the existing measures approach multivariate dependence through TC by a decomposition into multiple small variable sets\footnote{Throughout this work, we use $[n]\coloneqq \{1,2,\cdots,n\}$ and $[n]\setminus\{i\}\coloneqq [n]\setminus i$. For brevity, we frequently abbreviate the variable set $\{\mathbf{y}^1, \mathbf{y}^2,\cdots,\mathbf{y}^n\}$ with $\mathbf{y}^{[n]}$, and $\{\mathbf{y}^1,\cdots,\mathbf{y}^{i-1}, \mathbf{y}^{i+1},\cdots, \mathbf{y}^n\}$ with $\mathbf{y}^{[n]\setminus i}$.} (proof in supplementary material):
\begin{equation}\label{eq:TC_decomposition}
T(\mathbf{y}) = \sum_{i=1}^{L}{H(\mathbf{y}^i) - H(\mathbf{y}^i|\mathbf{y}^{[i-1]})}.
\end{equation}

In fact, these measures only vary in the way to estimate $H(\mathbf{y}^i)$ and $H(\mathbf{y}^i|\mathbf{y}^{[i-1]})$. For example, multivariate correlation~\cite{joe1989relative} and MAC~\cite{nguyen2014multivariate} use Shannon discrete entropy, whereas CMI~\cite{nguyen2013cmi} resorts to the cumulative entropy~\cite{rao2004cumulative} which can be directly applied on continuous variables. Although such progressive aggregation strategy helps a measure scales well to high dimensionality, it is sensitive to the ordering of the variables, i.e., Eq.~(\ref{eq:TC_decomposition}) is not permutation invariant. One should note that, there are a total of $L!$ possible permutations, which makes the decomposition scheme always achieve sub-optimal performances.

There are only a few exceptions avoid to the use of TC. A notable one is the Copula-based Kernel Dependence Measures (C-KDM)~\cite{poczos2012copula}, which instantiates $\mathsf{diff}\left(\right)$ in Eq.~(\ref{eq:diff}) with the Maximum Mean Discrepancy (MMD)~\cite{gretton2012kernel} and measures the discrepancy between $\mathsf{Pr}\left(\mathbf{y}^1,\mathbf{y}^2,\cdots,\mathbf{y}^L\right)$ and $\prod_{i=1}^{L}\mathsf{Pr}\left(\mathbf{y}^i\right)$ by first taking an empirical copular transform on both distributions. Although C-KDM is theoretically sound and permutation invariant, the value of C-KDM is not upper bounded, which makes it suffer from poor interpretability.


Last and not the least, the above mentioned measures can only deal with scalar variables. Thus, it still remains challenging when each variable is of an arbitrary dimension.

\subsection{Generalization of TC with Shearer's Inequality}
One should note that, TC is not the only non-negative measures of multivariate dependence. In fact, it can be seen as the simplest member of a large family, all obtained as special cases of an inequality due to Shearer~\cite{chung1986some}.

Given a set of $L$ random variables $\mathbf{y}^1,\mathbf{y}^2,\cdots,\mathbf{y}^L$. Denote $\varphi$ the family of all subsets of $[L]$ with the property that every member of $[L]$ lies in at least $k$ members of $\varphi$, the Shearer's inequality states that:
\begin{equation}\label{eq:Shearer}
H(\mathbf{y}^1,\mathbf{y}^2,\cdots,\mathbf{y}^L)\leq \frac{1}{k}\sum_{S\in\varphi} H(\mathbf{y}^i,i\in S).
\end{equation}
Obviously, TC (i.e., Eq.~(\ref{eq:TC_org})) is obtained when $\varphi={{L}\choose{1}}$.

Another important inequality arises when we take $\varphi={{L}\choose{L-1}}$, in which the Shearer's inequality suggests an alternative non-negative multivariate dependence measure as:
\begin{equation}\label{eq:DTC_inequality}
D(\mathbf{y}) = \left[\sum_{i=1}^L H(\mathbf{y}^{[L]\setminus i})\right] - (L-1)H(\mathbf{y}^1,\mathbf{y}^2,\cdots,\mathbf{y}^d).
\end{equation}

Eq.~(\ref{eq:DTC_inequality}) is also called the dual total correlation (DTC)~\cite{sun1975linear} and has an equivalent form~\cite{austin2018multi,abdallah2012measure} (see proof in supplementary material):

\begin{equation}\label{eq:DTC}
D(\mathbf{y}) = H(\mathbf{y}^1,\mathbf{y}^2,\cdots,\mathbf{y}^d) - \left[\sum_{i=1}^L H(\mathbf{y}^i | \mathbf{y}^{[L]\setminus i})\right].
\end{equation}

The Shearer¡¯s inequality suggests the existence of at least $(L-1)$ potential mathematical formulas to quantify the dependence of data, by just taking the gap between the two sides. Although all belong to the same family, these formulas emphasize different parts of the joint distributions and thus cannot be simply replaced by each other (see an illustrate figure in the supplementary material).
Finally, one should note that, the Shearer¡¯s inequality is just a loose bound on the sum of partial entropy terms. It has been recently refined further in~\cite{madiman2010information}. We leave a rigorous treatment to the tighter bound as future work.

\section{Matrix-based Dependence Measure}

\subsection{Our Measures and their Estimation}
We exemplify the use of Shearer's inequality in quantifying data dependence with TC and DTC in this work. First, to make TC and DTC more interpretable, i.e., taking values in the interval of $[0~1]$, we normalize both measures as follows:
\begin{equation}\label{eq:NTC}
\small
T^*(\mathbf{y})=\frac{\left[\sum_{i=1}^{L}{H(\mathbf{y}^i)}\right]-H(\mathbf{y}^1,\mathbf{y}^2,\cdots,\mathbf{y}^L)}{\left[\sum_{i=1}^{L}{H(\mathbf{y}^i)}\right]-\max\limits_{i}{H(\mathbf{y}^i)}},
\end{equation}

\begin{equation}\label{eq:NDTC}
\small
D^*(\mathbf{y})=\frac{\left[\sum_{i=1}^L H(\mathbf{y}^{[L]\setminus i})\right] - (L-1)H(\mathbf{y}^1,\mathbf{y}^2,\cdots,\mathbf{y}^L)}{H(\mathbf{y}^1,\mathbf{y}^2,\cdots,\mathbf{y}^L)}.
\end{equation}

Eqs.~(\ref{eq:NTC}) and (\ref{eq:NDTC}) involve entropy estimation in high-dimensional space, which is a notorious problem in statistics and machine learning~\cite{belghazi2018mutual}.
Although data discretization or entropy term decomposition has been used before to circumvent the ``curse of dimensionality", they all have their own intrinsic limitations. For data discretization, selecting a proper data discretization strategy is a tricky problem and an improper discretization may lead to serious estimation error. For entropy term decomposition, the resulting measure is no longer permutation invariant.


Different from earlier efforts, we introduce the recent proposed matrix-based R\'{e}nyi's $\alpha$-entropy functional~\cite{giraldo2014measures,yu2019multivariate}, which evaluate entropy terms in terms of the normalized eigenspectrum of the Hermitian matrix of the projected data in the reproducing kernel Hilbert space (RKHS), without explicit evaluation of the underlying data distributions. For brevity, we directly give the following definition.

\begin{defn}~\cite{giraldo2014measures} Let $\kappa:\mathcal{Y}\times\mathcal{Y}\mapsto\mathbb{R}$ be a real valued positive definite kernel that is also infinitely divisible~\cite{bhatia2006infinitely}. Given $Y=\{y_1,y_2,\cdots,y_N\}$, where the subscript $i$ denotes the exemplar index, and the Gram matrix $K$ obtained from evaluating a positive definite kernel $\kappa$ on all pairs of exemplars, that is $(K)_{ij}=\kappa(y_i,y_j)$, a matrix-based analogue to R{\'e}nyi's $\alpha$-entropy for a normalized positive definite (NPD) matrix $A$ of size $N\times N$,  such that $\mathrm{tr}(A)=1$, can be given by the following functional:
\begin{equation}\label{eq:renyi_entropy}
\mathbf{S}_\alpha(A)=\frac{1}{1-\alpha}\log_2\left(\mathrm{tr}(A^\alpha)\right)=
\frac{1}{1-\alpha}\log_2\big[\sum_{i=1}^N\lambda_i(A)^\alpha\big],
\end{equation}
where $A_{ij}=\frac{1}{N}\frac{K_{ij}}{\sqrt{K_{ii}K_{jj}}}$ and $\lambda_i(A)$ denotes the $i$-th eigenvalue of $A$.
\end{defn}

\begin{defn}~\cite{yu2019multivariate}
Given a collection of $N$ samples $\{s_i=(y_i^1,y_i^2,\cdots, y_i^L)\}_{i=1}^N$, each sample contains $L$ ($L\geq2$) measurements $y^1\in \mathcal{Y}^1$, $y^2\in \mathcal{Y}^2$, $\cdots$, $y^L\in \mathcal{Y}^L$ obtained from the same realization, and the positive definite kernels $\kappa_1:\mathcal{Y}^1\times \mathcal{Y}^1\mapsto\mathbb{R}$, $\kappa_2:\mathcal{Y}^2\times \mathcal{Y}^2\mapsto\mathbb{R}$, $\cdots$, $\kappa_L:\mathcal{Y}^L\times \mathcal{Y}^L\mapsto\mathbb{R}$, a matrix-based analogue to R{\'e}nyi's $\alpha$-order joint-entropy among $L$ variables can be defined as:
\begin{equation}\label{eq:renyi_joint}
\mathbf{S}_\alpha(A^{[L]})=\mathbf{S}_\alpha\left(\frac{A^1\circ A^2\circ\cdots\circ A^L}{\mathrm{tr}(A^1\circ A^2\circ\cdots\circ A^L)}\right),
\end{equation}
where $(A^1)_{ij}=\kappa_1(y_i^1,y_j^1)$, $(A^2)_{ij}=\kappa_2(y_i^2,y_j^2)$, $\cdots$, $(A^L)_{ij}=\kappa_L(y_i^L,y_j^L)$, and $\circ$ denotes the Hadamard product.
\end{defn}


Based on the above definition, we propose a pair of measures, namely the matrix-based normalized total correlation (denoted by $T_\alpha^*$) and the matrix-based normalized dual total correlation (denoted by $D_\alpha^*$):
\begin{equation}\label{eq:NTC_renyi}
T_{\alpha}^*(\mathbf{y}) = \frac{\left[\sum_{i=1}^{L}{S_{\alpha}(A^i)}\right]-S_{\alpha}\left(A^{[L]}\right)}{\left[\sum_{i=1}^{L}{S_{\alpha}(A^i)}\right] - \max\limits_{i}S_{\alpha}(A^i)},
\end{equation}

\begin{equation}\label{eq:NDTC_renyi}
D_{\alpha}^* (\mathbf{y}) =
\frac{\left[\sum_{i=1}^{L}S_\alpha\left(A^{[L]\setminus i}\right)\right]-(L-1)S_{\alpha}\left(A^{[L]}\right)}{S_{\alpha}\left(A^{[L]}\right)}.
\end{equation}

As can be seen, both $T_\alpha^*$ and $D_\alpha^*$ are independent of the specific dimensions of $\mathbf{y}^1, \mathbf{y}^2, \cdots, \mathbf{y}^L$ and avoid estimation of the underlying data distributions, which makes them suitable to be applied on data with either discrete or continuous distributions. Moreover, it is simple to verify that both $T_\alpha^*$ and $D_\alpha^*$ are permutation invariant to the ordering of variables.



\subsection{Properties and Observations of $T_{\alpha}^*$ and $D_{\alpha}^*$}\label{sec:property_measures}
We present more useful properties and observations of $T_\alpha^*$ and $D_\alpha^*$. In particular, we want to underscore that they are differentiable and can be used as loss functions to train neural networks. Note that, when $L=2$, both $T_\alpha^*$ and $D_\alpha^*$ reduce to the matrix-based normalized mutual information, which we denote by $I_\alpha^*$. See the supplementary material for proofs and additional supporting results.


\begin{property}
$0\leq T_{\alpha}^*\leq1$ and $0\leq D_{\alpha}^*\leq1$.
\end{property}

\begin{rem}
A major difference between our $T_\alpha^*$ and $D_\alpha^*$ to others is that our bounded property is satisfied with a finite number of realizations. An interesting and rather unfortunate fact is that although the statistics of many measures satisfies this desired property, their corresponding estimators hardly follow it~\cite{seth2012conditional}.
\end{rem}

\begin{property}
$T_{\alpha}^*$ and $D_{\alpha}^*$ reduce to zero iff $\mathbf{y}^1,\mathbf{y}^2,\cdots,\mathbf{y}^L$ are independent.
\end{property}

\begin{property}
$T_{\alpha}^*$ and $D_{\alpha}^*$ have analytical gradients and are automatically differentiable.
\end{property}
\begin{rem}
This property complements the theory of the matrix-based R\'{e}nyi's $\alpha$-entropy functional~\cite{giraldo2014measures,yu2019multivariate}, as it opens the door to challenging machine learning problems involving neural networks (as will be illustrated later in this work).
\end{rem}

\begin{property}
The computational complexity of $T_\alpha^*$ and $D_\alpha^*$ are respectively $\mathcal{O}(LN^2) + \mathcal{O}(N^3)$ and $\mathcal{O}(LN^3)$, and grows linearly with the number of variables $L$.
\end{property}

\begin{rem}
In case of $L=2$, both $T_\alpha^*$ and $D_\alpha^*$ cost $\mathcal{O}(N^3)$ in time. As a reference, the computational complexity of HSIC is between $\mathcal{O}(N^2)$ and $\mathcal{O}(N^3)$~\cite{zhang2018large}. However, HSIC only applies for two variables and is not upper bounded. We leave reducing the complexity as future work. But the initial exploration results, shown in the supplementary material, suggest that we can reduce the complexity by taking the average of the estimated quantity over multiple random subsamples of size $K\ll N$.
\end{rem}


\begin{observation}
$T_\alpha^*$ and $D_\alpha^*$ are more statistically powerful than prevalent random vector association measures, like HSIC, dCov, KCCA and QMI\_CS, in identifying independence and discovering complex patterns between $\mathbf{y}^1$ and $\mathbf{y}^2$.
\end{observation}


We made this observation with the same test data as has been used in~\cite{josse2016measuring,gretton2008kernel}.

The first test data is generated as follows~\cite{gretton2008kernel}. First, we generate $N$ $i.i.d.$ samples from two randomly picked densities in the ICA benchmark densities~\cite{bach2002kernel}. Second, we mixed these random variables using a rotation matrix parameterized by an angle $\theta$, varying from $0$ to $\pi/4$. Third, we added $d-1$ extra dimensional Gaussian noise of zero mean and unit standard deviation to each of the mixtures. Finally, we multiplied each resulting vector by an independent random $d$-dimensional orthogonal matrix. The resulting random vectors are dependent across all observed dimensions.



The second test data is generated as follows~\cite{szekely2007measuring}. A matrix $Y^1\in\mathbb{R}^{N\times5}$ is generated from a multivariate Gaussian distribution with an identity covariance matrix. Then, another matrix $Y^2\in\mathbb{R}^{N\times5}$ is generated as $Y_{ml}^2=Y_{ml}^1\epsilon_{ml}$, $m=1,2,\cdots,N$, $l=1,2,\cdots,5$, where $\epsilon_{ml}$ are standard normal variables and independent of $Y^1$.

In each test data, we compare all measures with a threshold computed by sampling a surrogate of the null hypothesis $H_0$ based on shuffling samples in $\mathbf{y}^2$ with $100$ times. That is, the correspondences between $\mathbf{y}^1$ and $\mathbf{y}^2$ are broken by the random permutations. The threshold is the estimated quantile $1-\tau$ where $\tau$ is the significance level of the test (Type I error). If the estimated measure is larger than the computed threshold, we reject the null hypothesis and argue the existence of an association between $\mathbf{y}^1$ and $\mathbf{y}^2$, and vice versa.

We repeated the above procedure $500$ independent trials. Fig.~\ref{fig:power_test_random_vector} demonstrated the averaged acceptance rate of the null hypothesis $H_0$ (in test data I with respect to different rotation angle $\theta$) and the averaged detection rate of the alternative hypothesis $H_1$ (in test data II with respect to different number of samples).

Intuitively, in the first test data, a zero angle means the data are independent, while dependence becomes easier to detect as the angle increases to $\pi/4$. Therefore, a desirable measure is expected to have acceptance rate of $H_0$ nearly to $1$ at $\theta=0$. But the rate is expected to rapidly decaying as $\theta$ increases. In the second test data, a desirable measure is expected to always have a large detection rate of $H_1$ regardless of the number of samples.

\begin{figure}[] 
\centering     
\subfigure[decaying dependence]{\includegraphics[width=0.23\textwidth]{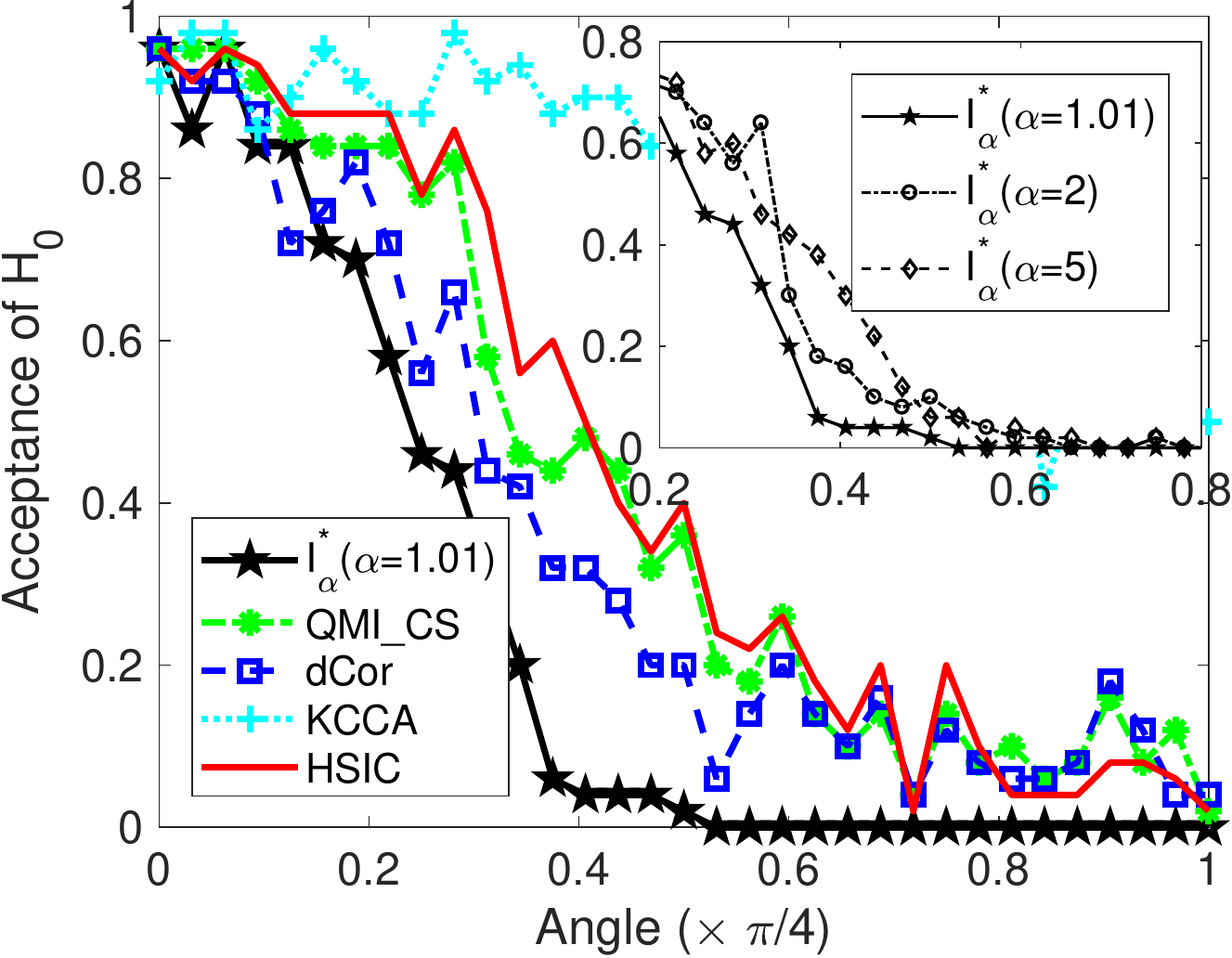}}
\subfigure[non-monotonic dependence]{\includegraphics[width=0.23\textwidth]{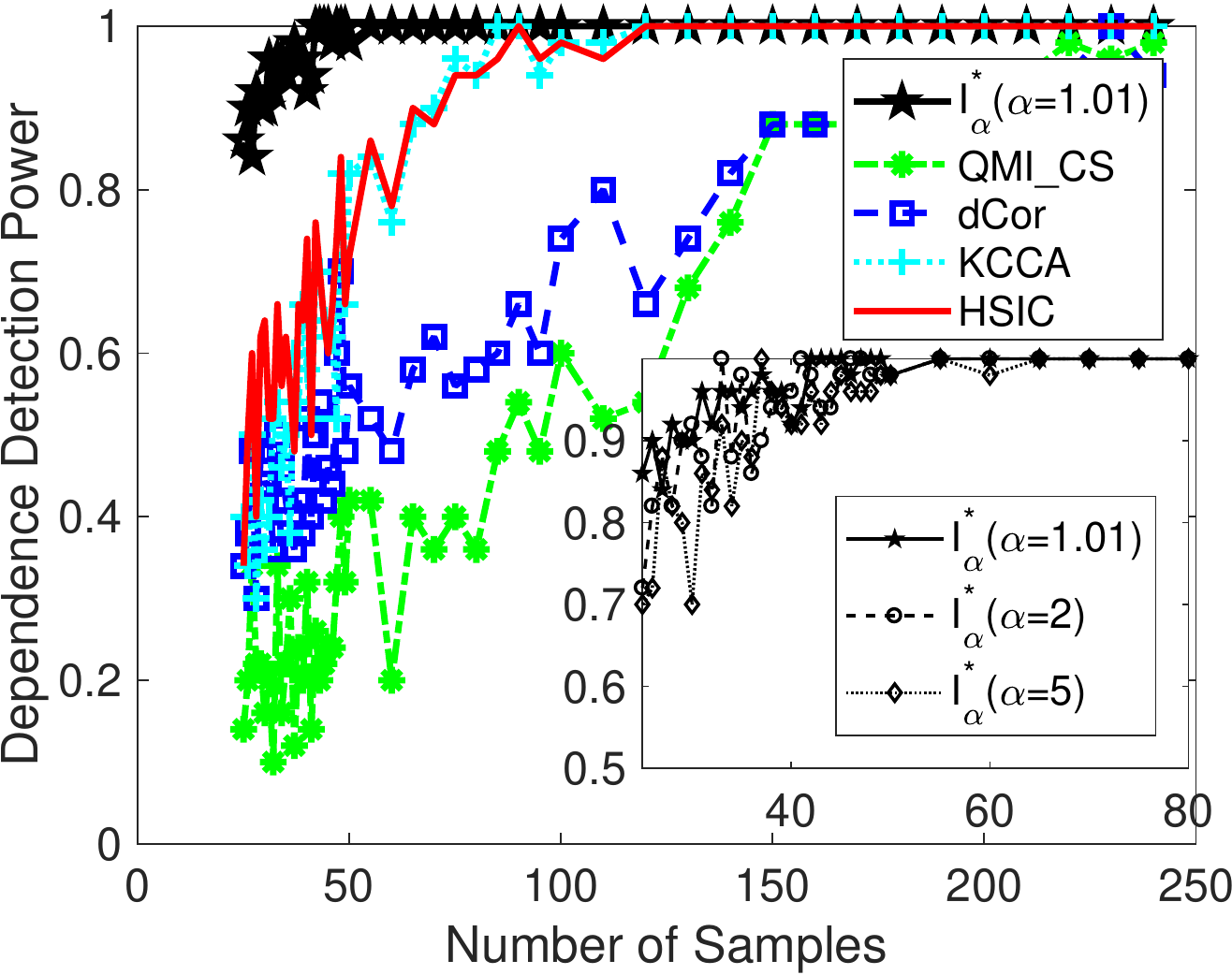}}
\caption{Power test against prevalent random vector association measures. Our measure is the most powerful one in a large range of $\alpha$.}
\label{fig:power_test_random_vector}
\end{figure}

\begin{observation}
$T_\alpha^*$ and $D_\alpha^*$ are more interpretable than their multivariate correlation counterparts in quantifying the dependence in each dimension of $\mathbf{y}=\{\mathbf{y}^1,\mathbf{y}^2,\cdots,\mathbf{y}^d\}\in \mathbb{R}^d$.
\end{observation}


This observation was made by comparing $T_\alpha^*$ and $D_\alpha^*$ against three popular multivariate correlation measures. They are multivariate Spearman's $\rho$, C-KDM and IDD~\cite{romano2016measuring}.
Fig.~\ref{fig:power_test_multivariate} shows the average value of the analyzed measures on the following relationships induced on $d\in [1,9]$ and $n=1000$ points:

\noindent
\textbf{Data~A}: The first dimension $\mathbf{y}^1$ is uniformly distributed in $[0,1]$, and $\mathbf{y}^i=(\mathbf{y}^1)^i$ for $i=2,3,\cdots,d$. The total dependence should be $1$, because $\{\mathbf{y}^2,\mathbf{y}^3,\cdots,\mathbf{y}^d\}$ depend nonlinearly only on $\mathbf{y}^1$.

\noindent
\textbf{Data~B}: There is a functional relationship between $\mathbf{y}^1$ and the remaining dimensions: $\mathbf{y}^1={(\frac{1}{d-1}\sum_{i=2}^{d}\mathbf{y}^i)}^2$, where $\{\mathbf{y}^2,\mathbf{y}^3,\cdots,\mathbf{y}^d\}$ are uniformly and independently distributed. In this case, the strength of the overall dependence should decrease with the increase of dimension.

\begin{figure}[htbp]
\setlength{\belowcaptionskip}{0pt}
\centering     
\subfigure[Data A]{\includegraphics[width=0.23\textwidth]{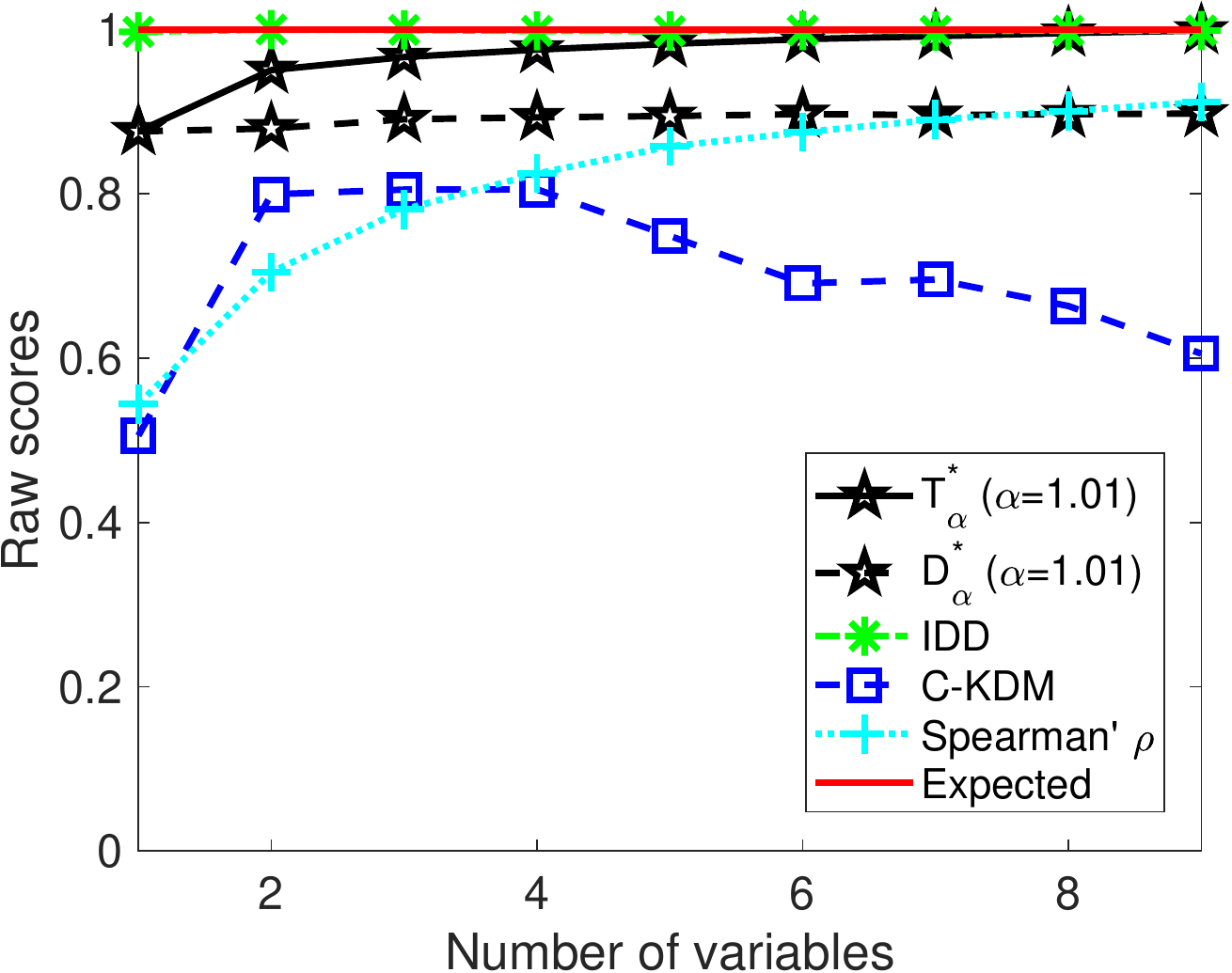}}
\subfigure[Data B]{\includegraphics[width=0.23\textwidth]{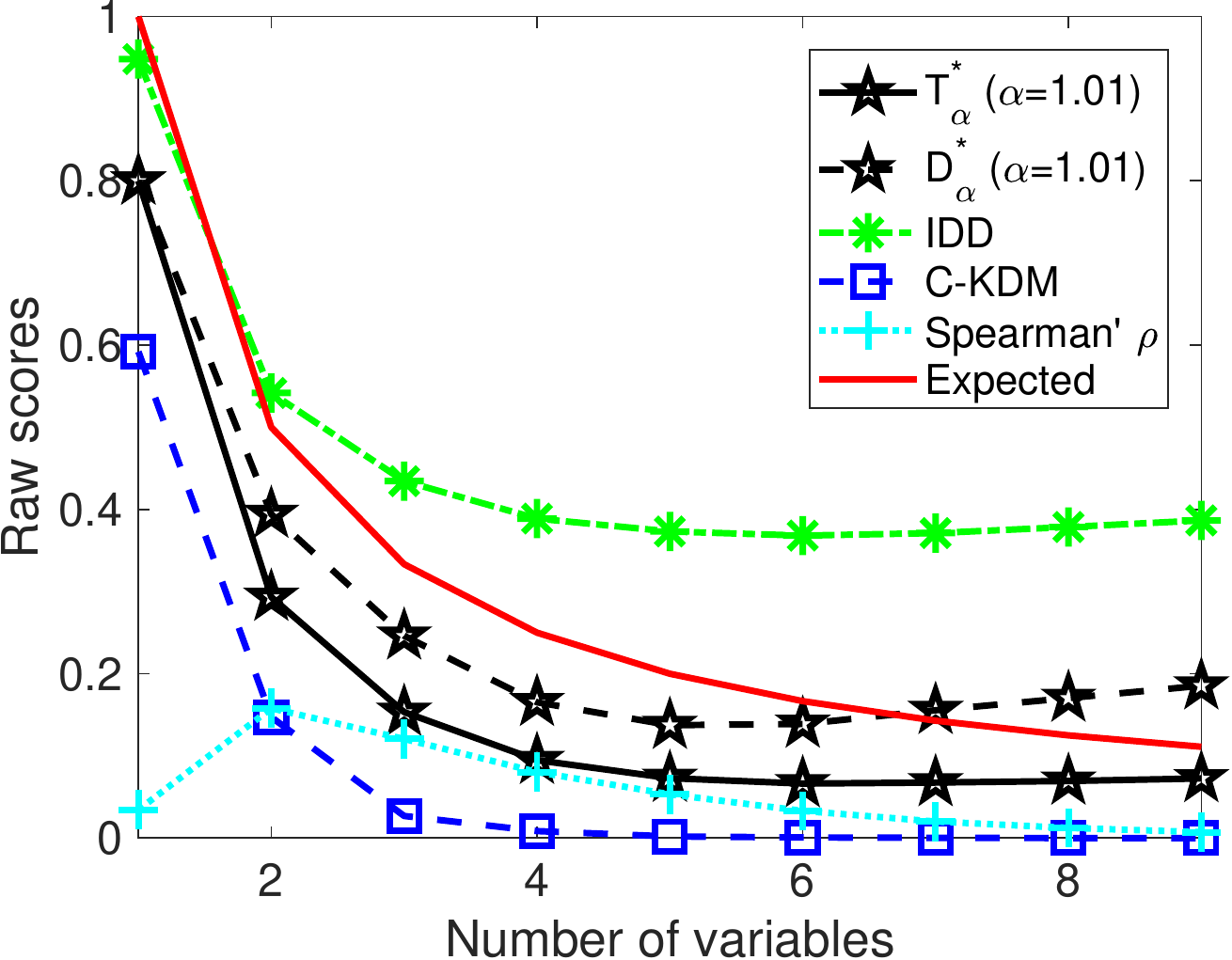}}
\caption{Raw measure scores on synthetic data with different relationships.}
\label{fig:power_test_multivariate}
\end{figure}

\section{Machine Learning Applications}

We present four solid machine learning applications to demonstrate the utility and superiority of our proposed matrix-based normalized total correlation ($T_\alpha^*$) and matrix-based normalized dual total correlation ($D_\alpha^*$). The applications include the gene regulatory network (GRN) inference, the robust machine learning under covariate shift and non-Gaussian noises, the subspace outlier detection and the understanding of the dynamics of learning of CNNs. The logic behind the organization of these applications is shown in Table~\ref{tab:four_scenarios}. We want to emphasize here that the use of normalization depends on the priority given to interpretability. For example, when the measure is employed as a loss function, the normalization does not contribute to performance. However, when we use it to quantify information flow or neural interactions in CNNs, a bounded value is preferred.


\begin{table}[]
\setlength{\abovecaptionskip}{1pt}
\begin{center}
\begin{tabular}{ c |c |c }
  & $d_i=1$ & $d_i\in \mathbb{Z}^+$ \\
\hline
 \rule{0pt}{25pt}$L=2$ & \shortstack{$\rho$, \emph{MIC} \\ \boxed{\text{GRN Inference}}} & \shortstack{ \emph{HSIC, dCov, QMI\_CS} \\ \boxed{\text{Robust ML}} } \\
\hline
  \rule{0pt}{25pt}$L>2$ & \shortstack{ \emph{C-KDM, IDD} \\ \boxed{\text{Outlier Detection}} } & \boxed{\text{Understanding CNNs}}
\end{tabular}
\end{center}
\caption{Four dependence scenarios. Popular measures in each scenario and potential applications.}
\label{tab:four_scenarios}
\end{table}

\subsection{Gene Regulatory Network Inference}

Gene expressions form a rich source of data from which to infer GRN, a sparse graph in which the nodes are genes and their regulators, and the edges are regulatory relationship between the nodes. In the first application, we resorted to the DREAM$4$ challenge~\cite{marbach2012wisdom} data set for reconstructing GRN. There are $5$ networks in the insilico (simulated) version of this data set, each contains expressions for $10$ genes with $136$ data points. The goal is to reconstruct the true network based on pairwise dependence between genes. We compared five test statistics (Pearsons¡¯ $\rho$, mutual information with respectively bin estimator and KSG estimator~\cite{kraskov2004estimating}, maximal information coefficient (MIC)~\cite{reshef2011detecting} and $I_\alpha^*$), and quantitatively evaluate reconstruction qualify by Area Under the ROC curve (AUC). Table~\ref{tab:gene} clearly indicates our superior performance.

\begin{table}[t]
\setlength{\abovecaptionskip}{0pt}
\setlength{\belowcaptionskip}{-10pt}
\begin{center}
\begin{tabular}{ c | c | c | c | c | c }
 Data Set & $\rho$ & MI (bin) & MI (KSG) & MIC & $I_\alpha^*$  \\
 \hline
 Network $1$ & 0.62 & 0.59 & 0.74 & $\underline{0.75}$ & $\mathbf{0.78}$ \\
 \hline
 Network $2$ & 0.52 & 0.58 & $\underline{0.76}$ & 0.74 & $\mathbf{0.87}$ \\
 \hline
 Network $3$ & 0.44 & 0.61 & $\underline{0.83}$ & 0.76 & $\mathbf{0.84}$  \\
 \hline
 Network $4$  & 0.45 & 0.60 & $\mathbf{0.75}$ & $\mathbf{0.75}$ & $\mathbf{0.75}$  \\
 \hline
 Network $5$  & 0.38 & 0.61 & 0.88 & $\underline{0.89}$ & $\mathbf{0.97}$
\end{tabular}
\end{center}
\caption{GRN inference results (AUC score) on DREAM$4$ challenge. The first and second best performances are in bold and underlined, respectively.}
\label{tab:gene}
\end{table}

\subsection{Robust Machine Learning}\label{sec:robust_ML}

Robust machine learning under domain shift~\cite{quionero2009dataset} has attracted increasing attentions in recent years. This is justified because the success of deep learning models is highly dependent on the assumption that the training and testing data are \emph{i.i.d.} and sampled from the same distribution. Unfortunately, the data in reality is typically collected from different but related domains~\cite{wilson2020survey}, and is corrupted~\cite{chen2016efficient}.

Let $(\mathbf{x},y)$ be a pair of random variables with $\mathbf{x}\in \mathbb{R}^p$ and $y\in \mathbb{R}$ (in regression) or $\mathbf{y}\in \mathbb{R}^q$ (in classification), such that $\mathbf{x}$ denotes input instance and $y$ denotes desired signal. We assume $\mathbf{x}$ and $y$ follow a joint distribution $P_{\text{source}}(\mathbf{x},y)$. Our goal is, given training samples drawn from $P_{\text{source}}(\mathbf{x},y)$, to learn a model $f$ predicting $y$ from $\mathbf{x}$ that works well on a different, a-priori unknown target distribution $P_{\text{target}}(\mathbf{x},y)$. We consider here only the covariate shift, in which the assumption is that the conditional label distribution is invariant (i.e., $P_{\text{target}}(y|\mathbf{x})=P_{\text{source}}(y|\mathbf{x})$) but the marginal distributions of input $P(\mathbf{x})$ are different between source and target domains (i.e., $P_{\text{target}}(\mathbf{x})\neq\ P_{\text{source}}(\mathbf{x})$). On the other hand, we also assume that $y$ (in the source domain) may be contaminated with non-Gaussian noises (i.e., $\widetilde{y}=y+\epsilon$). We focus on a fully unsupervised environment, in which we have no access to any samples $\mathbf{x}$ or $y$ from the target domain, i.e., the source-to-target manifold alignment becomes burdensome.

Our work in this section is directly motivated by~\cite{greenfeld2020robust}, which introduces the criterion of minimizing the dependence between input $\mathbf{x}$ and prediction error $e=y-f(\mathbf{x})$ to circumvent the covariate shift, and uses the HSIC as the measure to quantify the independence.
We provide two contributions over~\cite{greenfeld2020robust}. In terms of methodology, we show that by replacing HSIC with our new measures (i.e., $I_\alpha^*$), we improve the prediction accuracy in the target domain. Theoretically, we show that our new loss, namely $\min I_\alpha^*(\mathbf{x};e)$ is not only robust against covariate shift and also non-Gaussian noises on $y$ based on Theorem~\ref{th:equivalence}.


\begin{thm}\label{th:equivalence}
Minimizing the (normalized) mutual information $I(\mathbf{x};e)$ is equivalent to minimizing error entropy $H(e)$.
\end{thm}

\begin{rem}
The minimum error entropy (MEE) criterion~\cite{erdogmus2002error} has been extensively studied in signal processing to address non-Gaussian noises with both theoretical guarantee and empirical evidence~\cite{chen2009information,chen2016insights}. We summarize in supplementary material two insights to further clarify its advantage.
\end{rem}

\subsubsection{Learning under covariate shift.}
We first compare the performances of cross entropy (CE) loss, HSIC loss with our error entropy $H_{\alpha}(e)$ loss and $I_{\alpha}^*(\mathbf{x};e)$ loss under covariate shift.
Following~\cite{greenfeld2020robust}, the source data is the Fashion-MNIST dataset \cite{xiao2017fashion}, and images which are rotated by an angle $\theta$ sampled from a uniform distribution over $[-20\degree,20\degree]$ constitute the target data. The neural network architecture is set as: there are $2$ convolutional layers (with, respectively, $16$ and $32$ filters of size $5\times 5$) and $1$ fully connected layers. We add batch normalization and max-pooling layer after each convolutional layer. We choose ReLU activation, batch size $128$ and the Adam optimizer~\cite{kingma2014adam}.


For $I_{\alpha}^*(\mathbf{x};e)$ and $H_{\alpha}(e)$, we set $\alpha=2$. For the HSIC loss, we take the same hyper-parameters as in~\cite{greenfeld2020robust}. The results are summarized in Table~\ref{tab:fashion_mnist_comparison}. Our $H_{\alpha}(e)$ performs comparably to HSIC, but our $I_{\alpha}(\mathbf{x};e)$ improves performances in both source and target domains.


\begin{table}[]
\setlength{\abovecaptionskip}{3pt}
 \centering
  \fontsize{10}{8}\selectfont
\begin{threeparttable}
    \begin{tabular}{ccc}
    \toprule
    \multirow{2}{*}{Method}&
    \multicolumn{2}{c}{Fashion MNIST}\cr
    \cmidrule(lr){2-3}
    &Source&Target\cr
    \midrule
    CE&$90.90\pm 0.002$&$73.73\pm 0.086$\cr
    HSIC &$91.03\pm 0.003$&$76.56\pm 0.034$\cr
    \midrule
    $H_{\alpha}(e)$&$91.10\pm 0.013$&$75.48\pm 0.069$\cr
    $I_{\alpha}^*(\mathbf{x};e)$& $\mathbf{91.17\pm 0.040}$ & $\mathbf{76.79\pm 0.040}$ \cr
    \bottomrule
    \end{tabular}
    \end{threeparttable}
    \caption{Test accuracy (\%) on Fashion-MNIST}
    \label{tab:fashion_mnist_comparison}
\end{table}

\subsubsection{Learning in noisy environment.}
We select the widely used bike sharing data set~\cite{fanaee2014event} in UCI repository, in which the task is to predict the number of hourly bike rentals based on the following features: holiday, weekday, workingday, weathersit, temperature, feeling temperature, wind speed and humidity. Consisting of $17,379$ samples, the data was collected over two years, and can be partitioned by year and season. Early studies suggest that this data set contains covariate shift due to the change of time~\cite{subbaswamy2019preventing}.


We use the first three seasons samples as source data and the forth season samples as target data. The model of choice is a multi-layered perceptron (MLP) with three hidden layer of size 100, 100 and 10 respectively.
We compare our $I_\alpha^*(\mathbf{x};e)$ and $H_\alpha(e)$ with mean square error (MSE), mean absolutely error (MAE) and HSIC
loss, assuming $y$ is contaminated with additive noise as $\widetilde{y} = y + \epsilon$.
We consider two common non-Gaussian noises with the noise level controlled by parameter $\rho$: the Laplace noise $\epsilon \sim \textmd{Laplace}(0,\rho)$; and the shifted exponential noise $\epsilon = \rho(1-\eta)$ with $\eta \sim \textmd{exp}(1)$. We use batch-size of $32$ and the Adam optimizer.


We compared our $I_{\alpha}^*(\mathbf{x};e)$ and $H_{\alpha}(e)$ against MSE loss, MAE loss and HSIC loss. Fig.~\ref{fig:noise_comparison} demonstrates the averaged performance gain (or loss) of different loss functions over MSE loss in $10$ independent runs. In most of cases, $I_{\alpha}^*(\mathbf{x};e)$ improves the most. HSIC is not robust to Laplacian noise, whereas MAE performs poorly under shifted exponential noise. On the other hand, $H_{\alpha}(e)$ also obtained a consistent performance gain over MSE, which further corroborates our theoretical arguments.

\begin{figure}[htbp]
\setlength{\abovecaptionskip}{0pt}
\centering
\subfigure[Laplacian] {
    \includegraphics[width=0.225\textwidth]{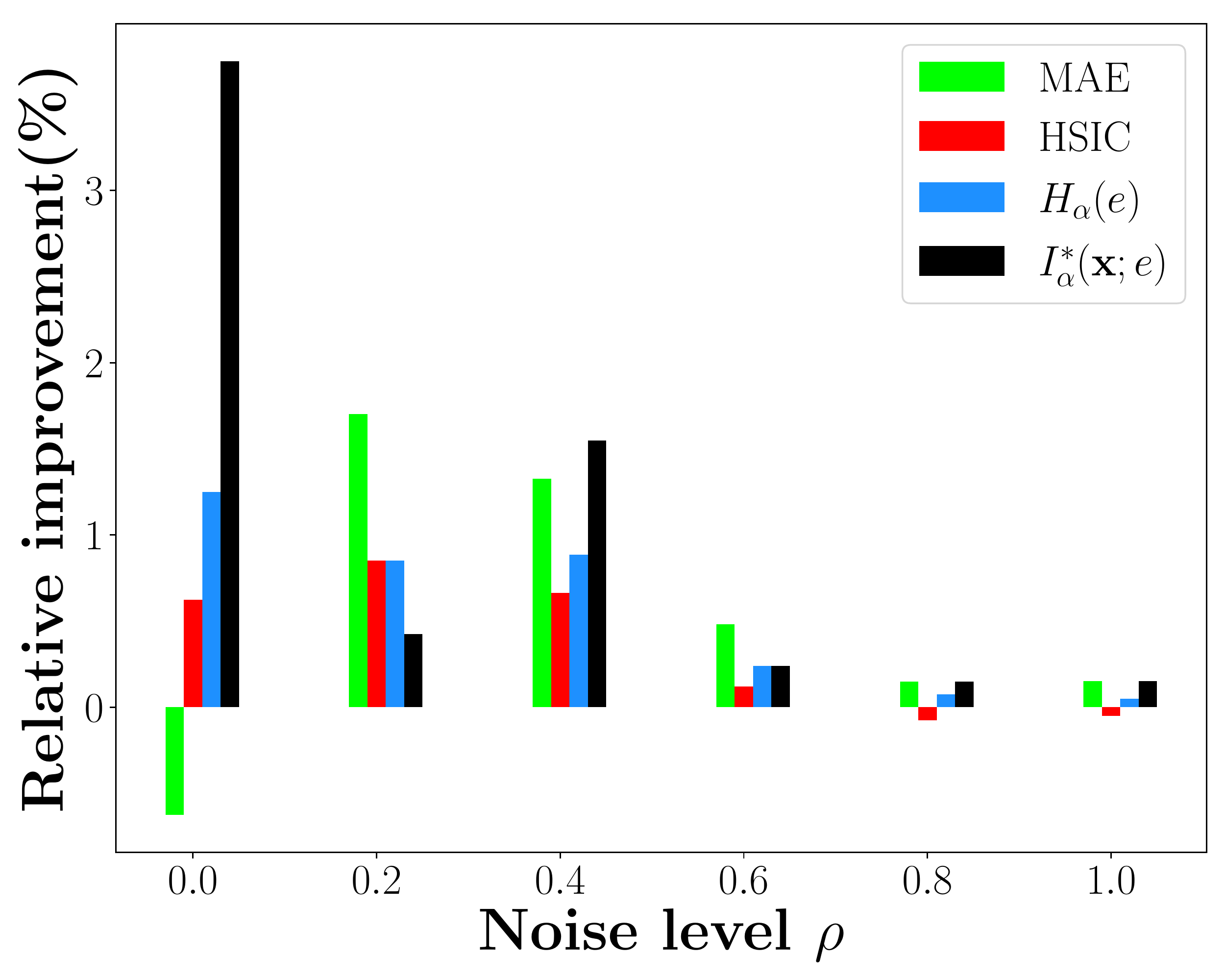}}
\subfigure[Shifted exponential] {
    \includegraphics[width=0.225\textwidth]{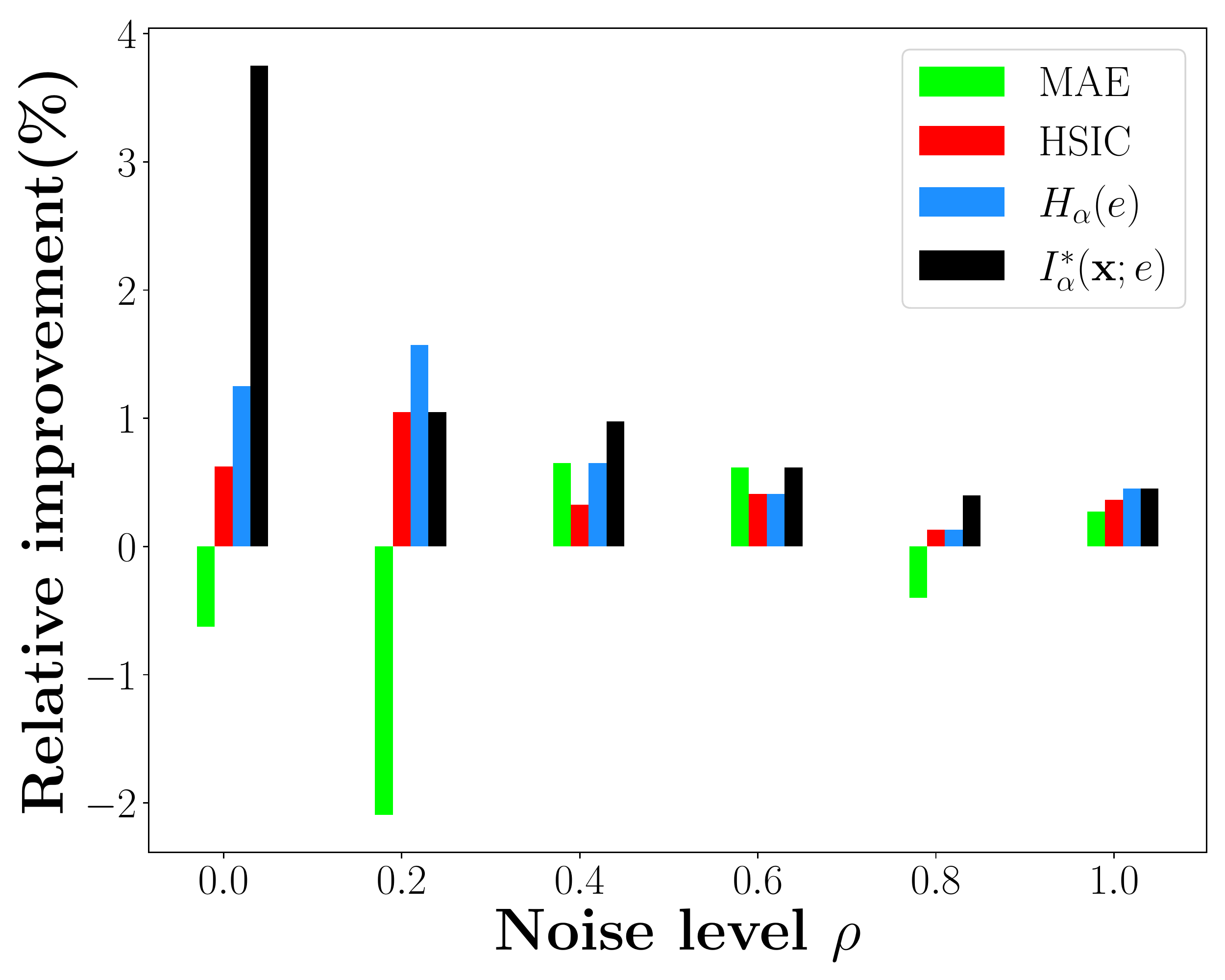}}
\caption{Comparisons of models trained with MSE, MAE, HSIC loss, $I_{\alpha}^*(\mathbf{x};e)$ and $H_{\alpha}(e)$. Each bar denotes the relative performance gain (or loss) over MSE.}
\label{fig:noise_comparison}
\end{figure}

\subsection{Subspace Outlier Detection}

Our third application is the outlier detection, in which we aim to identify data objects that do not fit well with the general data distributions (in $\mathbb{R}^d$). Despite diverse paradigms, such as the density-based methods~\cite{breunig2000lof} and the distance-based methods~\cite{bay2003mining}, have been developed so far, they usually suffer from the notorious ``curse of dimensionality"~\cite{keller2012hics}. In fact, the principle of \emph{concentration of distance}~\cite{beyer1999nearest} reveals that for a query point $p$, its relative distance (or contrast) $D$ to the farthest point and the nearest point converges to $0$ with the increase of dimensionality $d$:
\begin{equation}
\lim_{d\rightarrow\infty}{\frac{D_{max}-D_{min}}{D_{min}}}\rightarrow 0.
\end{equation}
This means that the discriminative power between the nearest and the farthest neighbor becomes rather poor in high-dimensional space.

On the other hand, real data often contains irrelevant attributes or noises. This phenomenon degrades further the performance of most existing outlier detection methods if the outliers are hidden in subspaces of all given attributes~\cite{kriegel2008angle}. Therefore, the subspace methods that explore lower-dimensional subspace in order to discover outliers provide a promising avenue.

Empirical evidence suggests that, the larger the deviation of this subspace from the mutual independence in each dimension, the higher the potential that it is easier to distinguish outliers from normal observations~\cite{muller2009relevant}. Therefore, measuring the total amount of dependence of a subspace becomes a pivotal aspect.
To this end, we plug our dependence measure (either $T_\alpha^*$ or $D_\alpha^*$) into a commonly used Apriori subspace search scheme~\cite{nguyen2013cmi} to assess the quality of each subspaces (the larger the better). Next, we detect outliers with a widely-used Local Outlier Factor (LOF) method~\cite{breunig2000lof} on the top $10$ subspaces with highest dependence score.


We use again the AUC to quantitatively evaluate outlier detection results of our method against three competitors: 1) LOF in full-space; 2) Feature Bagging (FB)~\cite{lazarevic2005feature} that applies LOF on randomly selected subspaces; and 3) LOF on subspaces generated by IDD. We omit the results of LOF on the subspaces generated by C-KDM due to relatively poor performance. We test on $5$ publicly available data sets from the Outlier Detection DataSets (ODDS) library~\cite{Rayana:2016}. These data cover a wide range of number of samples ($N$) and data dimensionality ($d$). The prevalence of anomalies ranges from $1.65\%$ (in speech) to $35\%$ (in breastW). The results are reported in Table~\ref{tab:outlier}. As can be seen, our $T_{\alpha}^*$ and $D_{\alpha}^*$ achieve remarkable performance gain over LOF on all attributes, especially when the $d$ is large. Under a Wilcoxon signed rank test~\cite{demvsar2006statistical} with $0.05$ significance level, both $T_{\alpha}^*$ and $D_{\alpha}^*$ significantly outperform LOF. This observation corroborates our motivation of reliable subspace search. By contrast, the random subspace selection scheme in FB does not show obvious advantage, and the subspace quality generated by IDD is lower than ours.

\begin{table}[t]
\setlength{\abovecaptionskip}{0pt}
\begin{center}
\begin{tabular}{ c | c | c | c | c | c }
 Data Set ($N\times d$) & LOF & FB & IDD & $T_\alpha^*$ & $D_\alpha^*$ \\
 \hline
 Diabetes ($568\times8$) & $\mathbf{0.68}$ & 0.63 & 0.55 & $\mathbf{0.68}$ & $\mathbf{0.68}$	 \\
 \hline
 breastW ($683\times9$) & 0.46 & 0.53 & $\mathbf{0.76}$ & $\underline{0.71}$ & $\underline{0.71}$ \\
 \hline
 Cardio ($1831\times21$) & 0.68 & 0.65 & 0.62 & $\mathbf{0.75}$ & $\underline{0.70}$ \\
 \hline
 Musk ($3062\times166$) & 0.42 & 0.40 & 0.67 & $\underline{0.73}$ & $\mathbf{0.83}$ \\
 \hline
 Speech ($3686\times400$) & 0.36 & 0.38 & 0.50 & $\mathbf{0.58}$ & $\underline{0.54}$
\end{tabular}
\end{center}
\caption{Outlier detection results (AUC score) on real data. The first and second best performances are in bold and underlined, respectively.}
\label{tab:outlier}
\end{table}

\subsection{Understanding the Dynamics of Learning of CNNs}

Understanding the dynamics of learning of deep neural networks (especially CNNs) has received increasing attention in recent years~\cite{shwartz2017opening,saxe2018information}. From an information-theoretic perspective, most studies aim to unveil fundamental properties associated with the dynamics of learning of CNNs by monitoring the mutual information between pairwise layers across training epochs~\cite{yu2020understanding}.

Different from the layer-level dependence, we provide here an alternative way to quantitatively analyze the dynamics of learning in a feature-level. Specifically, suppose there are $N_t$ feature maps in the $t$-th convolutional layer. Let us denote them by $C^1,C^2,\cdots, C^{N_t}$. We use two quantities to capture the dependence in feature maps: 1) the pairwise dependence between the $i$-th feature map and the $j$-th feature map (i.e., $I_{\alpha}^*(C^i;C^j)$); 2) the total dependence among all feature maps (i.e., $T_{\alpha}^*(C^1,C^2,\cdots,C^{N_t})$).

We train a standard VGG-$16$~\cite{simonyan2015very} on CIFAR-$10$~\cite{Krizhevsky09learningmultiple} with SGD optimizer from scratch. The $T_\alpha^*$ in different layers across different training epochs is illustrated in Fig.~\ref{fig:TC_CNN}. There is an obvious increasing trend for $T_\alpha^*$ in all layers during the training, i.e., the total amount of dependence amongst all feature maps continuously increases as the training moves on, until approaching to the value of nearly $1$. A similar observation is also made by HSIC. Note that, the co-adaptation phenomenon has also been observed in fully connected layers and eventually inspired the Dropout~\cite{hinton2012improving}.

Fig.~\ref{fig:pairwise_CNN} shows the histogram of $I_\alpha^*$ in each layer. Similar to the general trend of $T_\alpha^*$, we observed that the most frequent values of $I_\alpha^*$ change from nearly $0$ to nearly $1$. Moreover, such movement in lower layers occurs much earlier than that in upper layers. This behavior is in line with~\cite{raghu2017svcca}, which states that the neural networks first train and stabilize lower layers and then move to upper layers.



\begin{figure}[]
\setlength{\abovecaptionskip}{0pt}
\centering
\includegraphics[width=6cm]{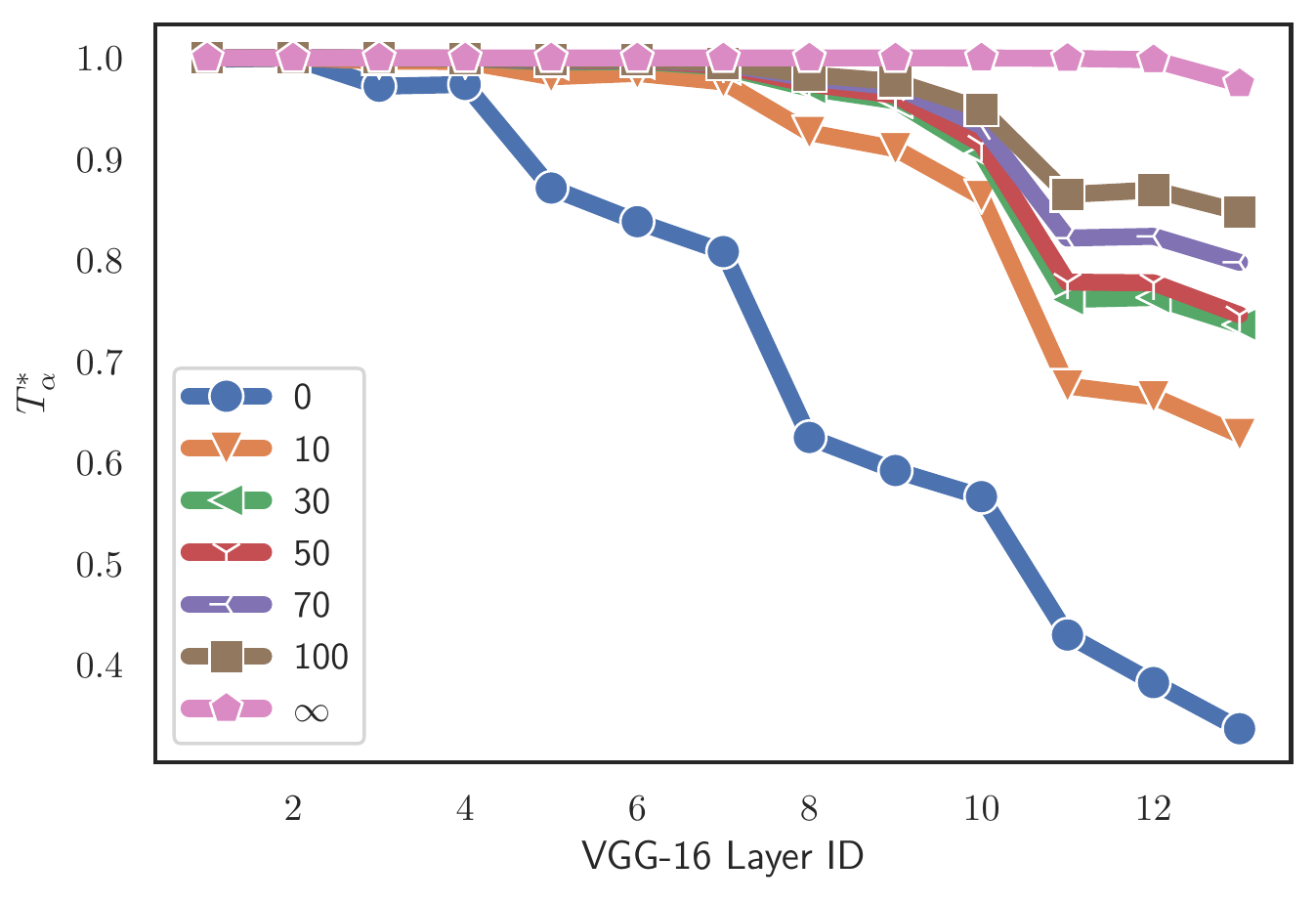}
\caption{$T_\alpha^*$ across training epochs for different convolutional layers.}
\label{fig:TC_CNN}
\end{figure}


\begin{figure}[]
\centering
\includegraphics[width=8.5cm]{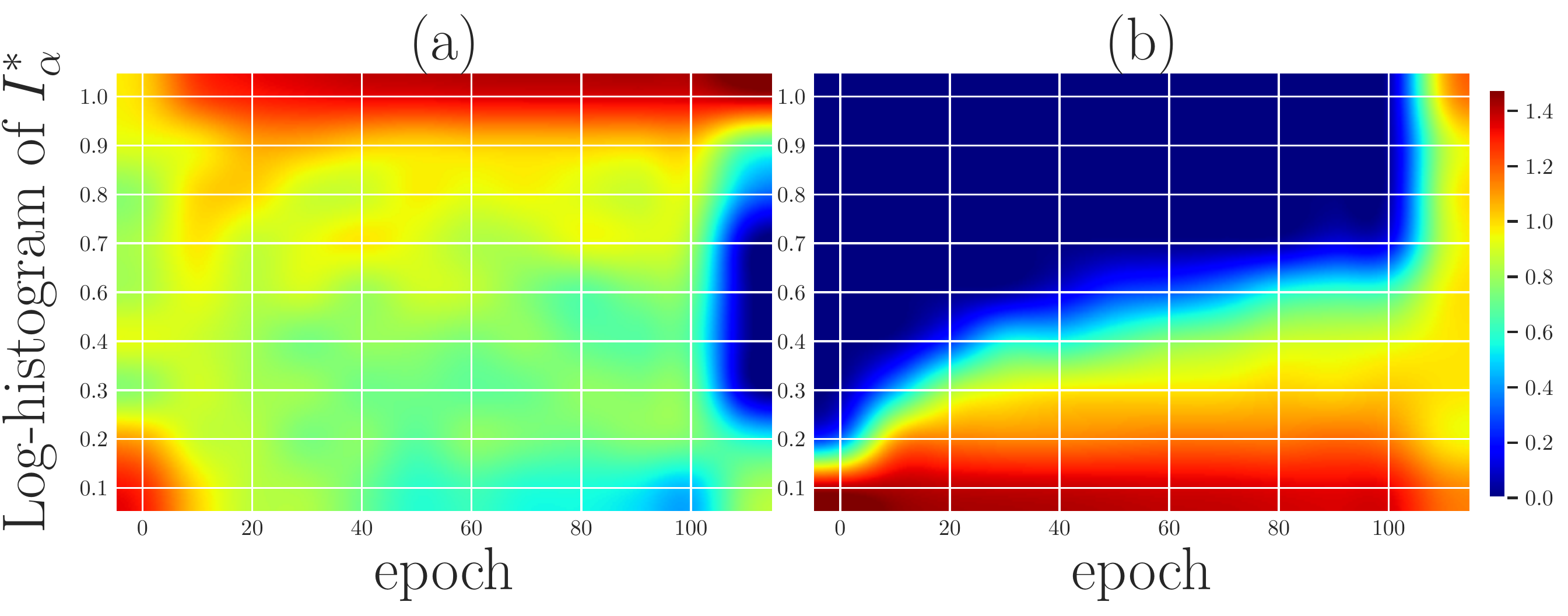}
\caption{The histogram of $I_\alpha^*$ (in log-scale) in (a) averaged from the $1$st to the $7$th CONV layers; and (b) averaged from the $8$th to $13$th CONV layer. Feature maps reach high dependence with less than $20$ epochs of training in lower layers, but need more than $100$ epochs in upper layers.}
\label{fig:pairwise_CNN}
\end{figure}

\section{Conclusion}


We suggest two measures to quantify from data the dependence of multiple variables with arbitrary dimensions.
Distinct from previous efforts, our measures avoid the estimation of the data distributions and are applicable to all dependence scenarios (for \emph{i.i.d.} data).
The proposed measures more easily (e.g., with less data) identify independence and discover complex dependence patterns. Moreover, the differentiable property enables us to design new loss functions for training neural networks.


In terms of specific applications, we demonstrated that the new loss $\min I_\alpha^*(\mathbf{x};e)$ is robust against both covariate shift and non-Gaussian noises. We also provided an alternative way to analyze the dynamics of learning of CNNs based on the dependence amongst feature maps, and obtained meaningful observations.

In the future, we will explore other properties of our measures. We are interested in applying them to other challenging problems, such as disentangled representation learning with variational autoencoders (VAEs)~\cite{Kingma2014AutoEncodingVB}. \textcolor{blue}{We also performed a preliminary investigation on a new robust loss, termed the deep deterministic information bottleneck (DIB), in the supplementary material.}

\section{Acknowledgement}
This work was funded in part by the Research Council of Norway grant no. 309439 SFI Visual Intelligence and grant no. 302022 DEEPehr, and in part by the U.S. ONR under grant N00014-18-1-2306 and DARPA under grant FA9453-18-1-0039.

\bibliography{dependence}

\newpage

\appendix
\section{Additional Note on $L=2$}
When $L=2$ (i.e., only two random variables), both total correlation (TC) and dual total correlation (DTC) reduce to the mutual information $I$:
\begin{equation}\label{eq:MI_unnormalize}
T(\mathbf{y})=D(\mathbf{y})=I(\mathbf{y})=H(\mathbf{y}^1)+H(\mathbf{y}^2)-H(\mathbf{y}^1,\mathbf{y}^2),
\end{equation}

One can normalize mutual information with either:
\begin{equation}
I^*(\mathbf{y})=\frac{H(\mathbf{y}^1)+H(\mathbf{y}^2)-H(\mathbf{y}^1,\mathbf{y}^2)}{\min\limits_{i}{H(\mathbf{y}^i)}},
\end{equation}
or
\begin{equation}\label{eq:MI_max}
I^*(\mathbf{y})=\frac{H(\mathbf{y}^1)+H(\mathbf{y}^2)-H(\mathbf{y}^1,\mathbf{y}^2)}{\max\limits_{i}{H(\mathbf{y}^i)}},
\end{equation}

In practice, we observed Eq.~(\ref{eq:MI_max}) always performs better.

\section{TC and DTC in Venn Diagram}
Although both total correaltion (TC) and dual total correaltion (DTC) reduce to zero in case of pairwise independence, they emphasize different parts of data distribution. When using a Venn diagram to represent a set of four variables $\mathbf{y}^1,\mathbf{y}^2,\mathbf{y}^3,\mathbf{y}^4$ as shown in Fig.~\ref{fig:illustration}, it is easy to find that TC gives more weights to higher-order interactions (see the dense block areas).

\begin{figure}[htbp]
\setlength{\abovecaptionskip}{0pt}
\setlength{\belowcaptionskip}{0pt}
\centering
\subfigure[TC or $T_\alpha^*$] {
    \includegraphics[width=4cm]{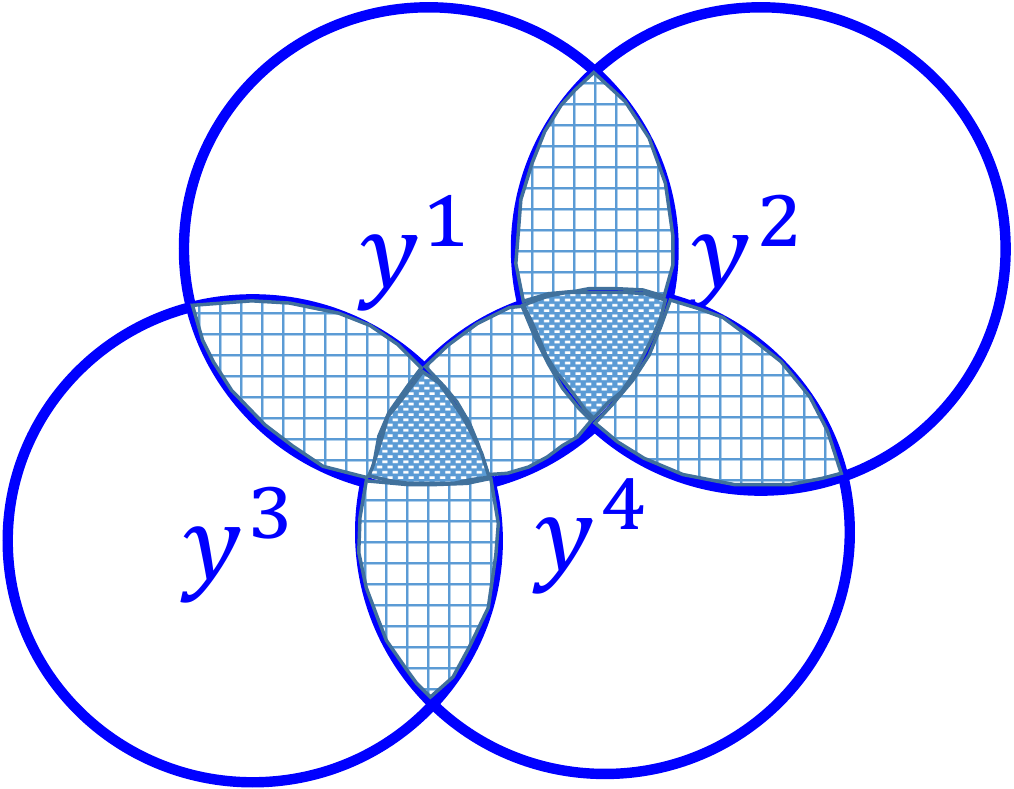}}
\subfigure[DTC or $D_\alpha^*$] {
    \includegraphics[width=4cm]{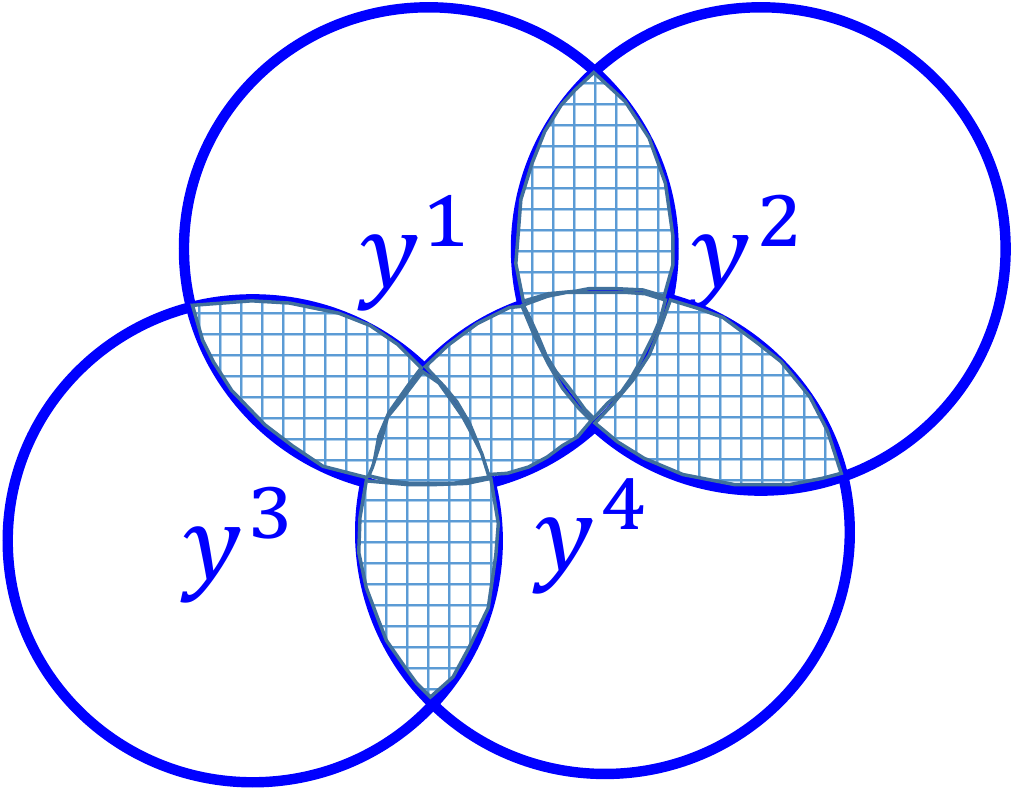}}
\caption{Illustration of TC (or $T_\alpha^*$) as compared with DTC (or $D_\alpha^*$) on a set of four random variables $\mathbf{y}^1,\mathbf{y}^2,\mathbf{y}^3,\mathbf{y}^4$. In each
case, the quantity is represented by the total amount of block areas of the diagram. TC counts the triple-overlapped areas twice each, whereas DTC counts each overlapped areas just once.}
\label{fig:illustration}
\end{figure}

\section{Proofs to Eqs.~(3) and~(6) of the Main Paper}
\subsection{Decomposition of Total Correlation}
By definition, we have:
\begin{equation}\label{eq:TC_org_supp}
    T(\mathbf{y})=\left[\sum_{i=1}^{L}{H(\mathbf{y}^i)}\right]-H(\mathbf{y}^1,\mathbf{y}^2,\cdots,\mathbf{y}^L),
\end{equation}

Eq.~(\ref{eq:TC_org_supp}) is equivalent to:
\begin{equation}\label{eq:TC_decomposition_supp}
T(\mathbf{y}) = \sum_{i=1}^{L}{H(\mathbf{y}^i) - H(\mathbf{y}^i|\mathbf{y}^{[i-1]})}.
\end{equation}

\begin{proof}
By the chain rule of joint entropy~\cite{cover1999elements}, we have:
\begin{equation}\label{eq:chain_rule_supp}
    H(\mathbf{y}^1,\mathbf{y}^2,\cdots,\mathbf{y}^L) = \sum_{i=1}^L H(\mathbf{y}^i|\mathbf{y}^{[i-1]}).
\end{equation}

Inserting Eq.~(\ref{eq:chain_rule_supp}) into Eq.~(\ref{eq:TC_org_supp}), we get Eq.~(\ref{eq:TC_decomposition_supp}).
\end{proof}

\subsection{Equivalent Expression of Dual Total Correlation}
By definition, we have:
\begin{equation}\label{eq:DTC_inequality_supp}
D(\mathbf{y}) = \left[\sum_{i=1}^L H(\mathbf{y}^{[L]\setminus i})\right] - (L-1)H(\mathbf{y}^1,\mathbf{y}^2,\cdots,\mathbf{y}^L).
\end{equation}

Eq.~(\ref{eq:DTC_inequality_supp}) is equivalent to:

\begin{equation}\label{eq:DTC_alternative_supp}
D(\mathbf{y}) = H(\mathbf{y}^1,\mathbf{y}^2,\cdots,\mathbf{y}^d) - \left[\sum_{i=1}^L H(\mathbf{y}^i | \mathbf{y}^{[L]\setminus i})\right].
\end{equation}

\begin{proof}
\begin{eqnarray}\label{eq:DTC_derive_supp}
D(\mathbf{y}) &=& \left[\sum_{i=1}^L H(\mathbf{y}^{[L]\setminus i})\right] - (L-1)H(\mathbf{y}^1,\mathbf{y}^2,\cdots,\mathbf{y}^L),\nonumber \\
&=& H(\mathbf{y}^1,\mathbf{y}^2,\cdots,\mathbf{y}^L) \nonumber \\
&-& \sum_{i=1}^L\left[H(\mathbf{y}^1,\mathbf{y}^2,\cdots,\mathbf{y}^L) - H(\mathbf{y}^{[L]\setminus i}) \right],\nonumber\\
&=& H(\mathbf{y}^1,\mathbf{y}^2,\cdots,\mathbf{y}^L) - \left[\sum_{i=1}^L H(\mathbf{y}^i | \mathbf{y}^{[L]\setminus i})\right].
\end{eqnarray}
\end{proof}

\section{Proofs and Additional Remarks to Properties}

\subsection{Elaboration of Property 1 and Property 2}

\begin{property}
$0\leq T_{\alpha}^*\leq1$ and $0\leq D_{\alpha}^*\leq1$.
\end{property}

\begin{proof}
By Corollary 2 in~\cite{yu2019multivariate}, we have:
\begin{equation}\label{eq:TC_corollary1_supp}
    S_\alpha(A^{[L]}) \leq
    \sum_{i=1}^{L}{S_\alpha(A^i)},
\end{equation}
and
\begin{equation}\label{eq:TC_corollary2_supp}
    S_\alpha(A^{[L]}) \geq \max\limits_{i} S_\alpha(A^i).
\end{equation}

Eqs.~(\ref{eq:TC_corollary1_supp}) and (\ref{eq:TC_corollary2_supp}) imply $0\leq T_{\alpha}^*\leq1$.

On the other hand, by Corollary 1 also in~\cite{yu2019multivariate}, we have:
\begin{equation}
    \forall i, S_\alpha(A^{[L]}) \geq \max\left[S_\alpha(A^i),S_\alpha(A^{[L]\setminus i}) \right],
\end{equation}
which implies:
\begin{equation}
    \forall i, S_\alpha(A^{[L]}) \geq S_\alpha(A^{[L]\setminus i}).
\end{equation}

Sum all $L$ inequalities, we obtain:
\begin{equation}
    L S_\alpha(A^{[L]}) \geq \sum_{i=1}^L S_\alpha(A^{[L]\setminus i}),
\end{equation}
which further implies that:
\begin{equation}
    \sum_{i=1}^L S_\alpha(A^{[L]\setminus i}) - (L-1)S_\alpha(A^{[L]}) \leq S_\alpha(A^{[L]})
\end{equation}
Therefore,
\begin{equation}
    D_\alpha^*(\mathbf{y}) = \frac{\left[\sum_{i=1}^{L}S_\alpha\left(A^{[L]\setminus i}\right)\right]-(L-1)S_{\alpha}\left(A^{[L]}\right)}{S_{\alpha}\left(A^{[L]}\right)} \leq 1.
\end{equation}

The non-negative of $D_\alpha^*$ or its numerator $D_\alpha$ is straightforward, in which we simply use the Lemma 1 (shown in property 2):
\begin{equation}
    D_\alpha\geq \max\limits_{i} I_\alpha(A^i;A^{[L]\setminus i}) \geq 0.
\end{equation}

\end{proof}

\begin{property}
$T_{\alpha}^*$ and $D_{\alpha}^*$ reduce to zero iff $\mathbf{y}^1,\mathbf{y}^2,\cdots,\mathbf{y}^L$ are independent.
\end{property}

Because the denominator of $T_{\alpha}^*$ and $D_{\alpha}^*$ is always larger than $0$, we focus our analysis on the numerator, i.e., the standard total correlation and dual total correlation. Our proof is based on the following lemma.

\begin{lem}
$T(\mathbf{y})$ and $D(\mathbf{y})$ are greater than or equal to $\max\limits_{i} I(\mathbf{y}^i;\mathbf{y}^{[L]\setminus i})$.
\end{lem}

\begin{proof}
For $T(\mathbf{y})$, we have (based on Eq.~(\ref{eq:TC_decomposition_supp})):
\begin{eqnarray}\label{eq:TC_lower}
T(\mathbf{y}) &=& \sum_{i=1}^{L}{H(\mathbf{y}^i) - H(\mathbf{y}^i|\mathbf{y}^{[i-1]})} \nonumber \\
& \geq & H(\mathbf{y}^L) - H(\mathbf{y}^L|\mathbf{y}^{[L-1]}) \nonumber\\
& = & I(\mathbf{y}^L;\mathbf{y}^{[L]\setminus L}),
\end{eqnarray}
in which the second line is based on the fact that $\forall{i}, H(\mathbf{y}^i) \geq H(\mathbf{y}^i|\mathbf{y}^{[i-1]})$ and we just take $i=L$.

For $D(\mathbf{y})$, we have (based on Eqs.~(\ref{eq:chain_rule_supp}) and (\ref{eq:DTC_alternative_supp})):
\begin{eqnarray}\label{eq:DTC_lower}
D(\mathbf{y}) & = & \sum_{i=1}^L H(\mathbf{y}^i|\mathbf{y}^{[i-1]}) - \sum_{i=1}^L H(\mathbf{y}^i | \mathbf{y}^{[L]\setminus i}) \nonumber \\
& = & \sum_{i=1}^L H(\mathbf{y}^i|\mathbf{y}^{[i-1]}) - H(\mathbf{y}^i | \mathbf{y}^{[L]\setminus i}) \nonumber \\
& = & \sum_{i=1}^L I(\mathbf{y}^i;\mathbf{y}^{i+1,\cdots,L} |\mathbf{y}^{[i-1]}) \nonumber \\
& \geq & I(\mathbf{y}^1;\mathbf{y}^{[L]\setminus 1}).
\end{eqnarray}
in which the second line is based on the fact that $\forall{i}, I(\mathbf{y}^i;\mathbf{y}^{i+1,\cdots,L} |\mathbf{y}^{[i-1]}) \geq 0$ and we just take $i=1$.

Note that Eqs.~(\ref{eq:TC_lower}) and (\ref{eq:DTC_lower}) can be applied for any re-ordering of the random variables $\mathbf{y}^1,\mathbf{y}^2,\cdots,\mathbf{y}^L$, i.e., invariant to the order of $1,2,\cdots,L$. This implies the common lower bound $\max\limits_{i} I(\mathbf{y}^i;\mathbf{y}^{[L]\setminus i})$.
\end{proof}

Note that, mutual information is non-negative. According to Lemma 1, in case of $T(\mathbf{y}) = 0$ or $D(\mathbf{y}) = 0$, we have $\forall{i}, I(\mathbf{y}^i;\mathbf{y}^{[L]\setminus i})=0$, which implies that for all $i$, $\mathbf{y}^i$ is independent of all rest variables $\mathbf{y}^1,\cdots,\mathbf{y}^{i-1},\mathbf{y}^{i+1},\cdots,\mathbf{y}^{L}$. Therefore, $T(\mathbf{y})$ and $D(\mathbf{y})$
reduce to zero iff $\mathbf{y}^1,\mathbf{y}^2,\cdots,\mathbf{y}^L$ are independent.

\subsection{Elaboration of Property 3}
\begin{property}
$T_{\alpha}^*$ and $D_{\alpha}^*$ have analytical gradients and are automatically differentiable.
\end{property}

\subsubsection{Mutual information}

We focus our analysis on the proposed loss $\min I_\alpha(\mathbf{x};e)$, i.e., encouraging the distribution of the prediction residuals $e=y-f_\theta(\mathbf{x})$ is statistically independent of the distribution of the input $\mathbf{x}$.
Mutual information can be equivalently expressed as the sum of each entropy term minus the joint entropy, i.e., $I_\alpha(A;B)=S_\alpha(A)+S_\alpha(B)-S_\alpha(A,B)$, in which $A\in \mathbb{R}^{N\times N}$ and $B \in \mathbb{R}^{N\times N}$ denote two sample Gram matrices generated from $\{\mathbf{x}_i\}_{i=1}^N$ and $\{e_i\}_{i=1}^N$, respectively.

According to Definition 1 and Definition 2, we have:
\begin{equation}
    S_\alpha(A) = \frac{1}{1-\alpha}\log_2\left(\mathrm{tr}(A^{\alpha})\right),
\end{equation}
and
\begin{equation}
    S_\alpha(A,B) = S_\alpha\left(\frac{A\circ B}{\mathrm{tr}(A \circ B)}\right).
\end{equation}

We thus have:
\begin{equation}
    \frac{\partial S_\alpha(A)}{\partial A} = \frac{\alpha}{(1-\alpha)} \frac{A^{\alpha-1}}{\mathrm{tr}\left(A^\alpha\right)},
\end{equation}

\begin{equation}
    \frac{\partial S_\alpha(A,B)}{\partial A} = \frac{\alpha}{(1-\alpha)}
    \left[ \frac{(A\circ B)^{\alpha-1}\circ B}{\mathrm{tr}(A\circ B)^{\alpha}}
    - \frac{I\circ B}{\mathrm{tr}(A\circ B)}
    \right],
\end{equation}
and
\begin{equation}
    \frac{\partial I_\alpha(A;B)}{\partial A} = \frac{\partial S_\alpha(A)}{\partial A} + \frac{\partial S_\alpha(A,B)}{\partial A}.
\end{equation}

Since $I_\alpha(A;B)$ is symmetric, the same applies for $\frac{\partial I_\alpha(A;B)}{\partial B} $ with exchanged roles between $A$ and $B$.

\subsubsection{Automatic differentiation}
In practice, taking the gradient of the $I_{\alpha}(A,B)$ is simple with any automatic differentiation software, like PyTorch~\cite{paszke2019pytorch} or Tensorflow~\cite{abadi2016tensorflow}. We use PyTorch in this work, because the obtained gradient is consistent with the analytical one. For example, by the Theorem 1 in~\cite{magnus1985differentiating}, the analytical gradient of the $i$-th eigenvalue with respect to a real symmetric matrix $A$ is the outer product of the $i$-th eigenvector ($v_i$), i.e.,:
\begin{equation}
    \frac{\partial\lambda_i}{\partial A}=v_i v_i^T.
\end{equation}
We noticed that, with Tensorflow, the diagonal entries are the same to their corresponding analytical values, but the off-diagonal entries are just half.

\subsection{Elaboration of Property 4}
\begin{property}
The computational complexity of $T_\alpha^*$ and $D_\alpha^*$ are respectively $\mathcal{O}(LN^2) + \mathcal{O}(N^3)$ and $\mathcal{O}(LN^3)$, and grows linearly with the number of variables $L$.
\end{property}

\begin{proof}
For $T_\alpha^*$, the first term is the complexity of computing $L$ Gram matrices with $N$ samples, while the second term is due to the eigenvalue decomposition of matrices of size $N \times N$.

Similarly, for $D_\alpha^*$, the main complexity comes from eigenvalue decomposition of $2L$ matrices of size $N\times N$.
\end{proof}

In practice, one is able to reduce the complexity by
taking the average of the estimated quantity over multiple
random subsamples of size $K\ll N$. Suppose we take $M$ groups of subsamples, then the total complexity reduces to $\mathcal{O}(LMK^2) + \mathcal{O}(MK^3)$ for $T_\alpha^*$ and $\mathcal{O}(LMK^3)$ for $D_\alpha^*$. Note that, this strategy is common in information-theoretic quantities estimation~\cite{wang2019fast,ver2014non} and our initial results (see Figs.~\ref{fig:subsample_independence} and \ref{fig:subsample_dataD}) suggest that the power and the interpretability of our measures do not suffer too much with this strategy.

\begin{figure}[hbt!]
\centering
\includegraphics[width=0.4\textwidth]{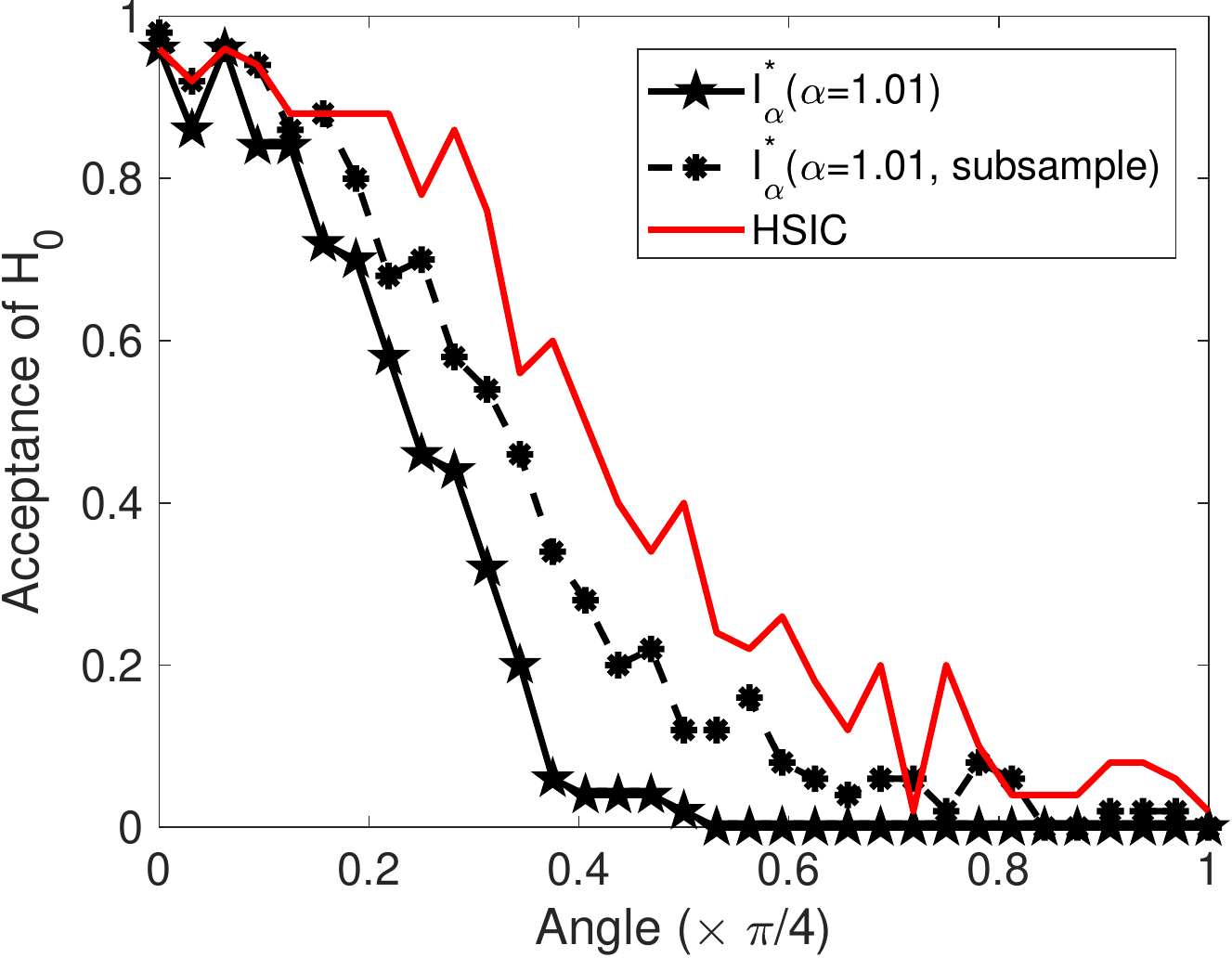}
\caption{Power test against HSIC in ``decaying dependence" data. $I_\alpha^*$ is obtained with $512$ samples, $I_\alpha^*$ (subsample) is obtained by taking the average of $10$ groups of subsamples of size $128$. Note that, although the dependence detection power of $I_\alpha^*$ decreases with subsampling, it is still superior to the established HSIC (with full samples).}
\label{fig:subsample_independence}
\end{figure}

\begin{figure}[hbt!]
\centering
\includegraphics[width=0.4\textwidth]{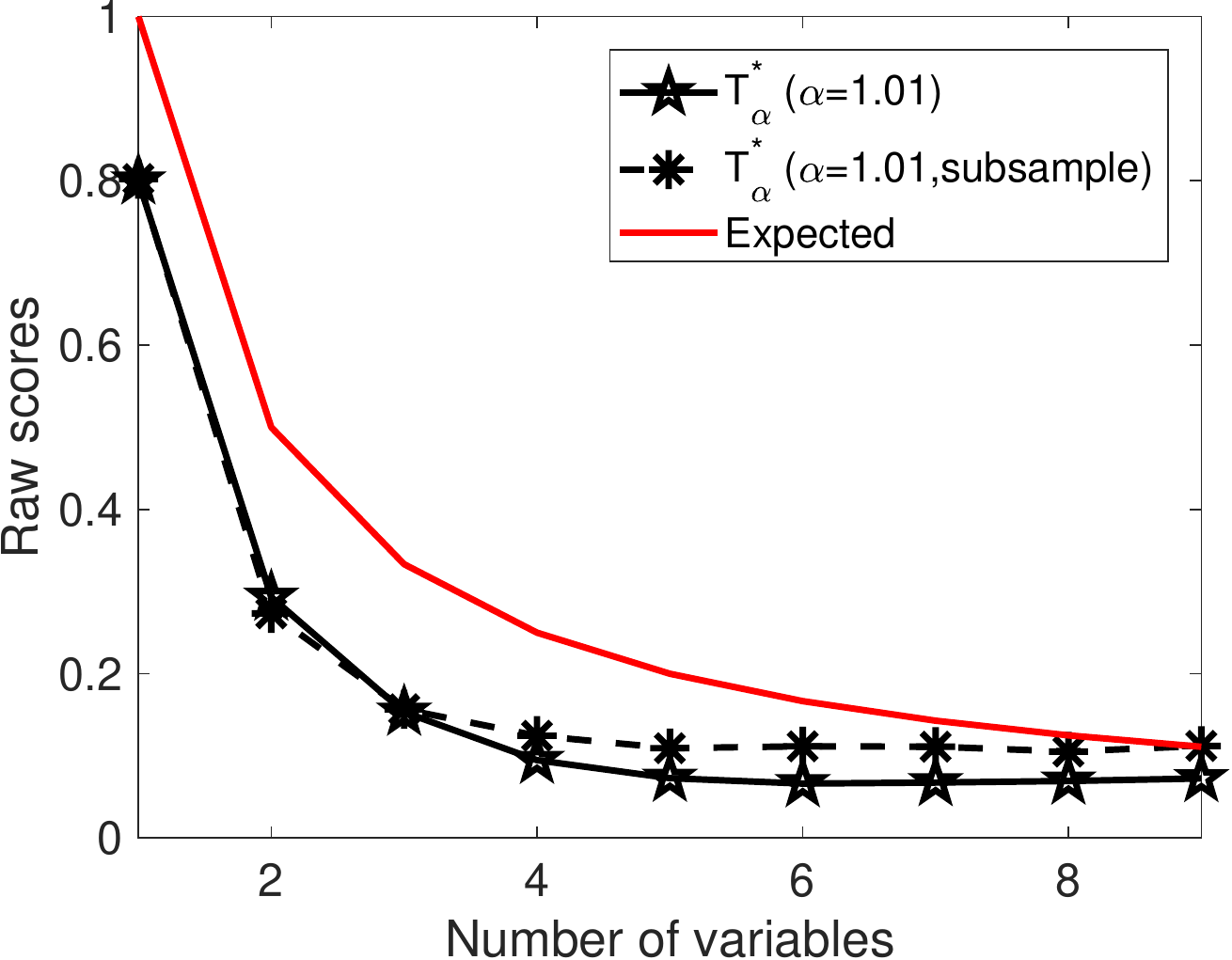}
\caption{Raw measure scores on Data B. $T_\alpha^*$ is obtained with $1000$ samples, $T_\alpha^*$ (subsample) is obtained by taking the average of $10$ groups of subsamples of size $100$. In both cases, the raw scores of $T_\alpha^*$ have similar values and fit well with the expected ones.}
\label{fig:subsample_dataD}
\end{figure}



\section{Proofs and Additional Remarks to Theorems}

\subsection{Proof of Theorem 1 of the Main Paper}

Here we prove the equivalence of $\min I(\mathbf{x};e)$ and $\min H(e)$, where the latter is the well-known minimum error entropy (MEE) criterion~\cite{erdogmus2002error} that has been extensively investigated in signal processing and process control.

\begin{thm}\label{th:equivalence_supp}
Minimizing the mutual information $I(\mathbf{x};e)$ is equivalent to minimizing error entropy $H(e)$.
\end{thm}

\begin{proof}
We have:
\begin{equation}
\label{eq_equality}
\begin{split}
I(\mathbf{x}; e) & = H(e) - H(e|\mathbf{x})\\
& = H(e) - H(e+f_\theta(\mathbf{x})|\mathbf{x}) \\
& = H(e) - H(y|\mathbf{x}),
\end{split}
\end{equation}
in which the second line is by the property that given two random variables $\xi$ and $\eta$, then for any measurable function $h$, we have $H(\xi|\eta)=H(\xi+h(\eta)|\eta)$. Interested readers can refer to~\cite{cover1999elements,mackay2003information} for a detailed account of this property and its interpretation.

Therefore, $\min I(\mathbf{x};e)$ is equivalent to $\min H(e)$. This is simply by the fact that the conditional entropy of $y$ given $\mathbf{x}$, i.e., $H(y|\mathbf{x})$, is a constant that is purely determined by the training data (regardless of training algorithms). Note that, similar argument has also been claimed in stochastic process control~\cite{feng1997optimal}.
\end{proof}

\subsection{Insights into $\min H(e)$ against non-Gaussian Noises}
We then present two insights on the robustness of $\min H(e)$ over the mean square error (MSE) criterion $\min E(e^2)$ against non-Gaussian noises, in which $E$ denotes the expectation. Interested readers can refer to~\cite{chen2009information,chen2010new,chen2016insights,hu2013learning} for more thorough analysis on the advantage of $\min H(e)$.

First,~\cite[Theorem~3]{chen2009information} suggests that $\min E(e^2)$ is equivalent to minimizing the error entropy plus a Kullback¨CLeibler (KL) divergence. Specifically, we have:
\begin{equation}\label{eq:mse_mee}
    \min E(e^2)\Leftrightarrow \min H(e) + D_{\text{KL}}(p(e)\|\varphi(e)),
\end{equation}
in which $p(e)$ is the probability of error, $\varphi(e)$ denotes a zero-mean Gaussian distribution. As the KL-divergence is always nonnegative, minimizing the MSE is equivalent to minimizing an upper bound of the error entropy. Eq.~(\ref{eq:mse_mee}) also explains (partially) why in linear and Gaussian cases, the MSE and MEE are equivalent~\cite{kalata1979linear}. Nevertheless, in case the error or noise follows a highly non-Gaussian distribution (especially when the signal-to-noise (SNR) value decreases), the MSE solution deviates from the MEE result, but the latter takes full advantage of high-order information of the error~\cite{chen2016insights}.

On the other hand, given the mean-square error $E(e^2)$, the error entropy satisfies~\cite{cover1999elements}:
\begin{equation}
    H(e)\le\max_{E(\zeta^2)=E(e^2)} H(\zeta)=\frac{1}{2}+\frac{1}{2}\log{2\pi}+\frac{1}{2}\log{(E(e^2))},
\end{equation}
where $\zeta$ denotes a random variable whose second moment equals to $E(e^2)$. This implies that the MSE criterion can be recognized as a robust MEE criterion in the minimax sense, because:
\begin{equation}
\label{eq_mee_upper}
\begin{split}
f_{\text{MSE}}^* & = \argmin_{f\in F} E(e^2) \\
& = \argmin_{f\in F} \frac{1}{2}+\frac{1}{2}\log{2\pi}+\log{(E(e^2))} \\
& = \argmin_{f\in F} \max_{E(\zeta^2)=E(e^2)}H(\zeta),
\end{split}
\end{equation}
where $f_{\text{MSE}}^*$ denotes the solution with MSE criterion, $\mathcal{F}$ stands for the collection of all Borel measurable functions. Eq.~(\ref{eq_mee_upper}) suggests that minimizing the MSE is equivalent to minimizing an upper bound of the error entropy.

\section{Additional Experimental Results}

\subsection{Robust Machine Learning}

For clarity, we summarize the general gradient-based method for $\min I_{\alpha}^*(\mathbf{x};e)$ in Algorithm~\ref{code:recentEnd}. Same to~\cite{greenfeld2020robust}, one weakness of this loss function is that $I_\alpha(\mathbf{x},y-f_{\theta}(\mathbf{x}))$ is exactly the same for any two functions $f_{\theta_1}(\mathbf{x})$, $f_{\theta_2}(\mathbf{x})$ who differ only by a constant $c$, i.e., $f_{\theta_1}(\mathbf{x}) = f_{\theta_2}(\mathbf{x})+c$. Thus, we calculate the empirical mean of error for training set and add this bias in the output of the model. We fix $\alpha=2$ in this section.

\begin{algorithm}[h]
	\caption{Learning with $I_\alpha^*(\mathbf{x};e)$}
	\begin{algorithmic}[1]
        \State  \textbf{Input:} samples $\{\mathbf{x}_i,y_i\}^n_{i=1}$, R{\'e}nyi's entropy order $\alpha$, mini-batch size $m$.
        \State Initialize neural network parameter $\theta$.
		\State \textbf{Repeat:}\\ \quad Sample mini-batch $\{\mathbf{x}_i,y_i\}^m_{i=1}$.\\
		\quad Evaluate the prediction residual for each instances in the mini-batch $e_{i} = y_{i}-f_{\theta}(x_i)$.\\
		\quad Compute the (normalized) Gram matrices of size $m\times m$ for $\{\mathbf{x}_i\}^m_{i=1}$ and $\{e_i\}^m_{i=1}$ (denote them $A_\mathbf{x}$ and $A_e$, respectively).\\
		\quad Compute the normalized R{\'e}nyi's $\alpha$-entropy mutual information (i.e., $I_{\alpha}^*(\mathbf{x},e)$) based on $A_\mathbf{x}$ and $A_e$.\\
		\quad Update $\theta \leftarrow \textrm{Optimize}(I_\alpha^*(\mathbf{x};e))$.
		
        \State \textbf{Until} convergence.\\
        Compute the estimated source bias: \\
        $b\leftarrow \frac{1}{n}\sum_{i=1}^{n}y_{i} - \frac{1}{n}\sum_{i=1}^{n}f_{\theta}(\mathbf{x}_i)$.
        \State \textbf{Output} $f(\mathbf{x})=f_{\theta}(\mathbf{x})+b$.
	\end{algorithmic}
	\label{code:recentEnd}
\end{algorithm}

We demonstrate here that our loss $I_\alpha^*(\mathbf{x};e)$ and $H_\alpha^*(e)$ also achieve appealing performance under Gaussian noise. Fig.~\ref{fig:Gaussian_noises} supports our argument. Only $I_\alpha^*(\mathbf{x};e)$ and $H_\alpha^*(e)$ perform better than MSE in heavy Gaussian noises.

\begin{figure}[htbp]
\setlength{\abovecaptionskip}{0pt}
\centering
\includegraphics[width=0.3\textwidth]{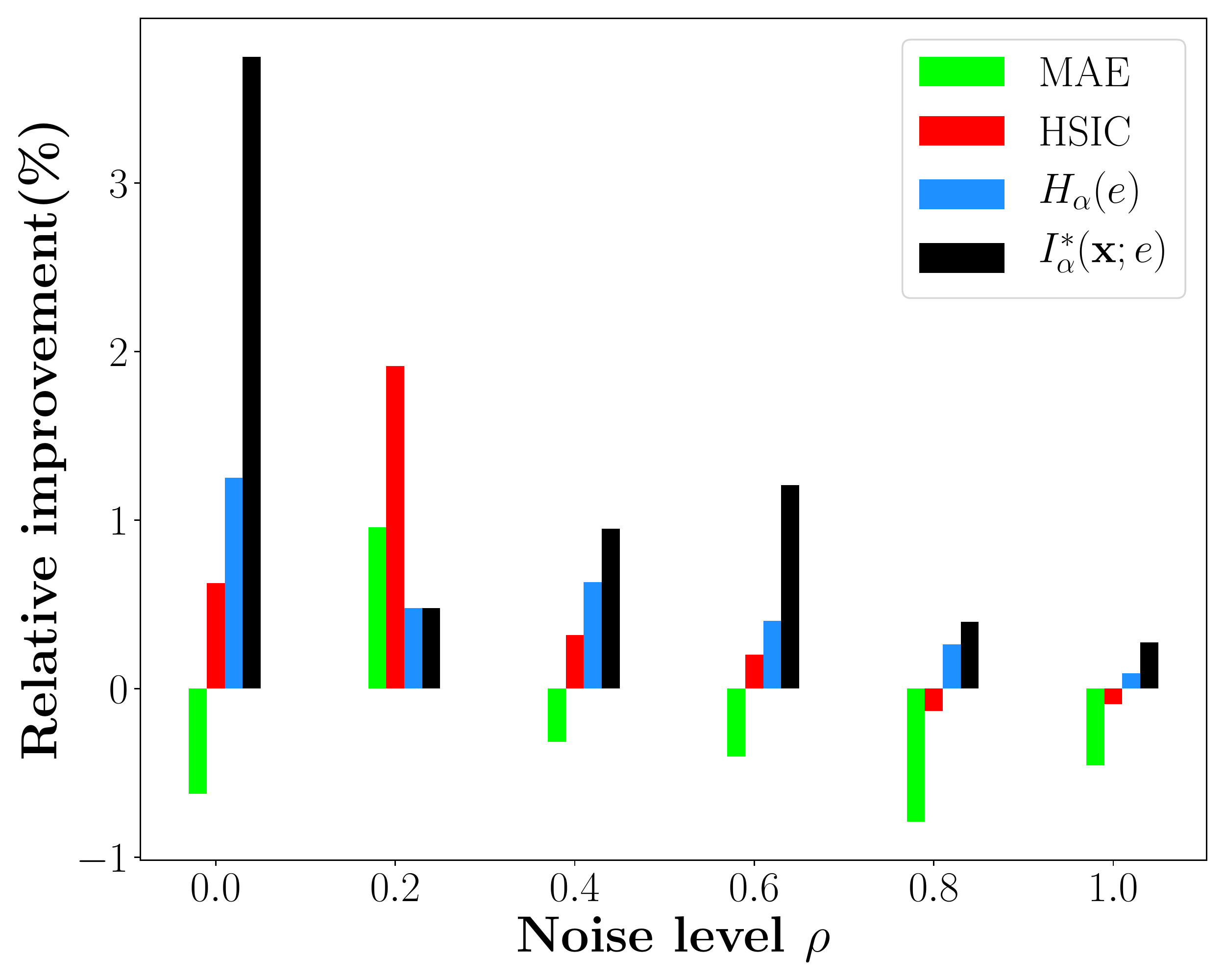}
\caption{Comparisons of models trained with MSE, MAE, HSIC loss, $I_{\alpha}^*(\mathbf{x};e)$ and $H_{\alpha}(e)$ under Gaussian noise $\epsilon \sim \mathcal{N}(0,\rho)$. Each bar denotes the relative performance gain (or loss) over MSE. Only $I_{\alpha}^*(\mathbf{x};e)$ and $H_{\alpha}(e)$ perform better than MSE under heavy Gaussian noises. }
\label{fig:Gaussian_noises}
\end{figure}

\subsection{Understanding the Dynamics of CNNs}


We show in the main paper that our measures $T_\alpha^*$ and $I_\alpha^*$ reveal two properties behind the training of CNNs: 1) the increase of total amount of dependence amongst all feature maps $C^1,C^2,\cdots, C^{N_t}$; 2) the movement and stabilization of pairwise dependence occur much earlier in lower layers, compared with that in upper layers.
In this section, we show that similar observations have also been made by HSIC (see Figs.~\ref{fig:TC_CNN_HSIC} and~\ref{fig:HSIC_CNN_100}). Note that, HSIC only applies for two random variables (here refers to two feature maps). In order to obtain the total amount of dependence, we follow the procedure in~\cite{gretton2005measuring} and take the average of the sum of all pairwise dependence, i.e., $\frac{1}{N_t(N_t-1)}\sum_{i=1}^N\sum_{j\neq i} \mathrm{HSIC}(C^i,C^j)$.


\begin{figure}[ht]
\setlength{\abovecaptionskip}{0pt}
\centering
\includegraphics[width=6cm]{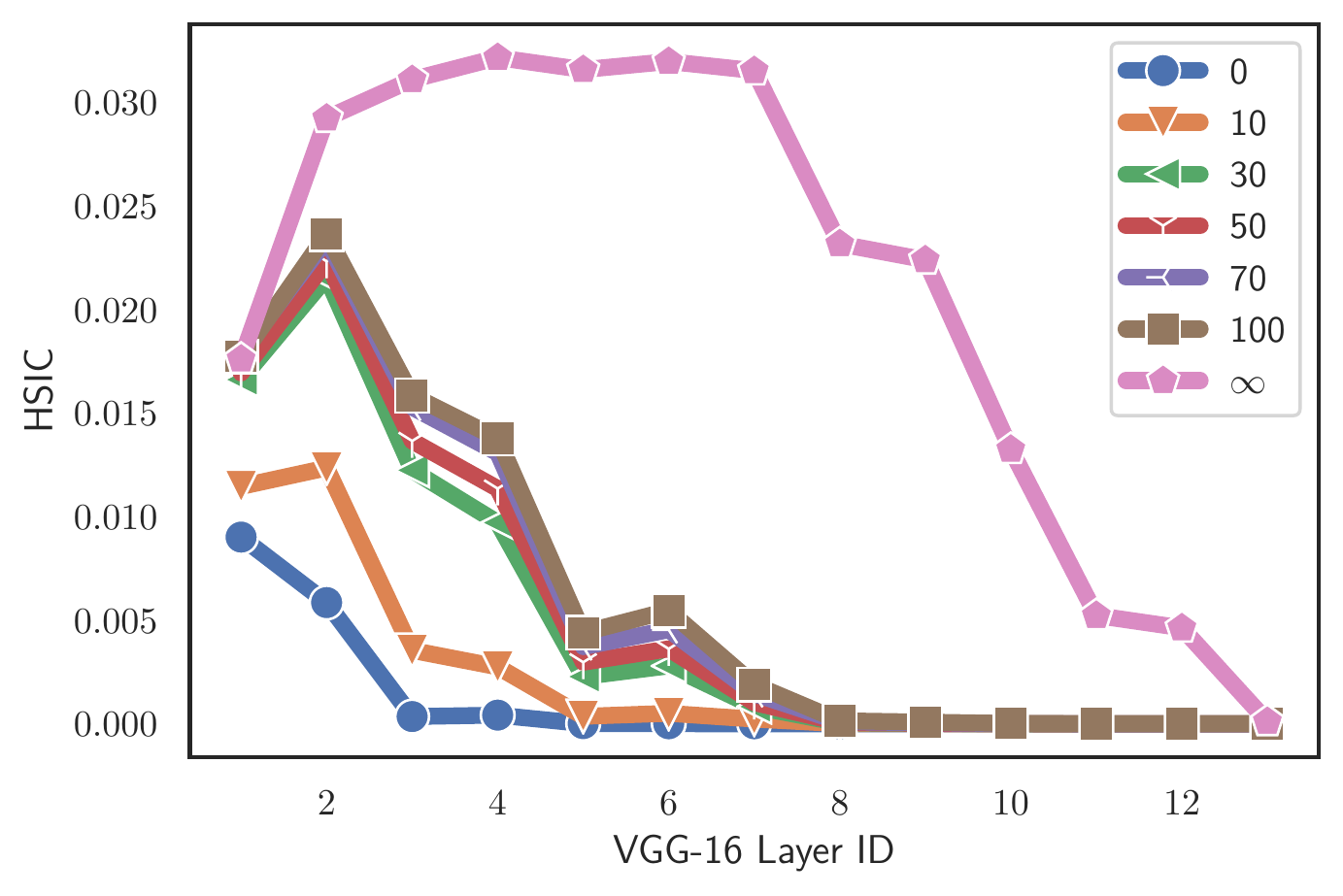}
\caption{Mean of the sum of all pairwise HSIC scores across training epochs for different convolutional layers.}
\label{fig:TC_CNN_HSIC}
\end{figure}

\begin{figure}[htp!]
\setlength{\abovecaptionskip}{0pt}
\setlength{\belowcaptionskip}{0pt}
\centering
\subfigure[] {
    \includegraphics[width=4cm]{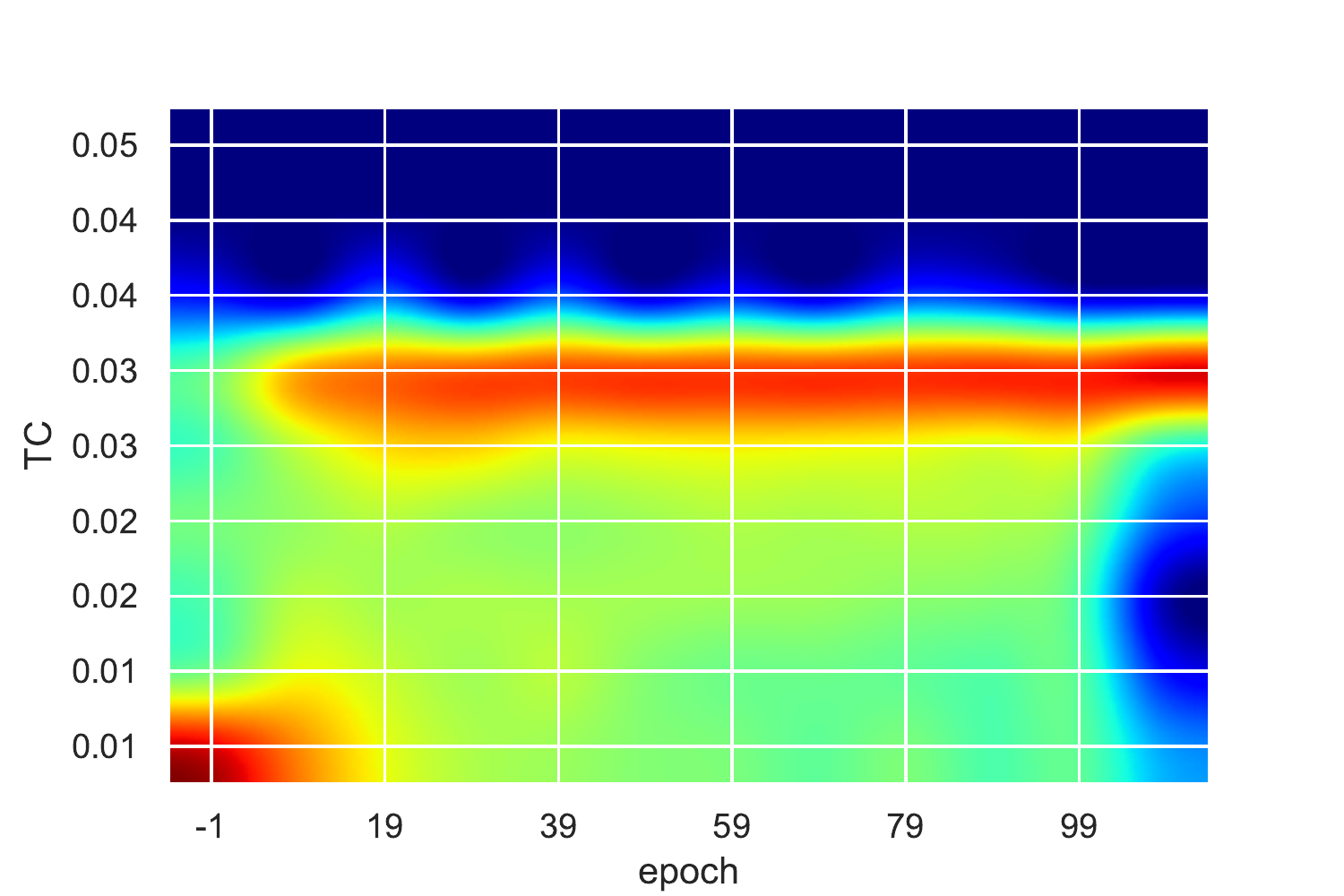}}
\subfigure[] {
    \includegraphics[width=4cm]{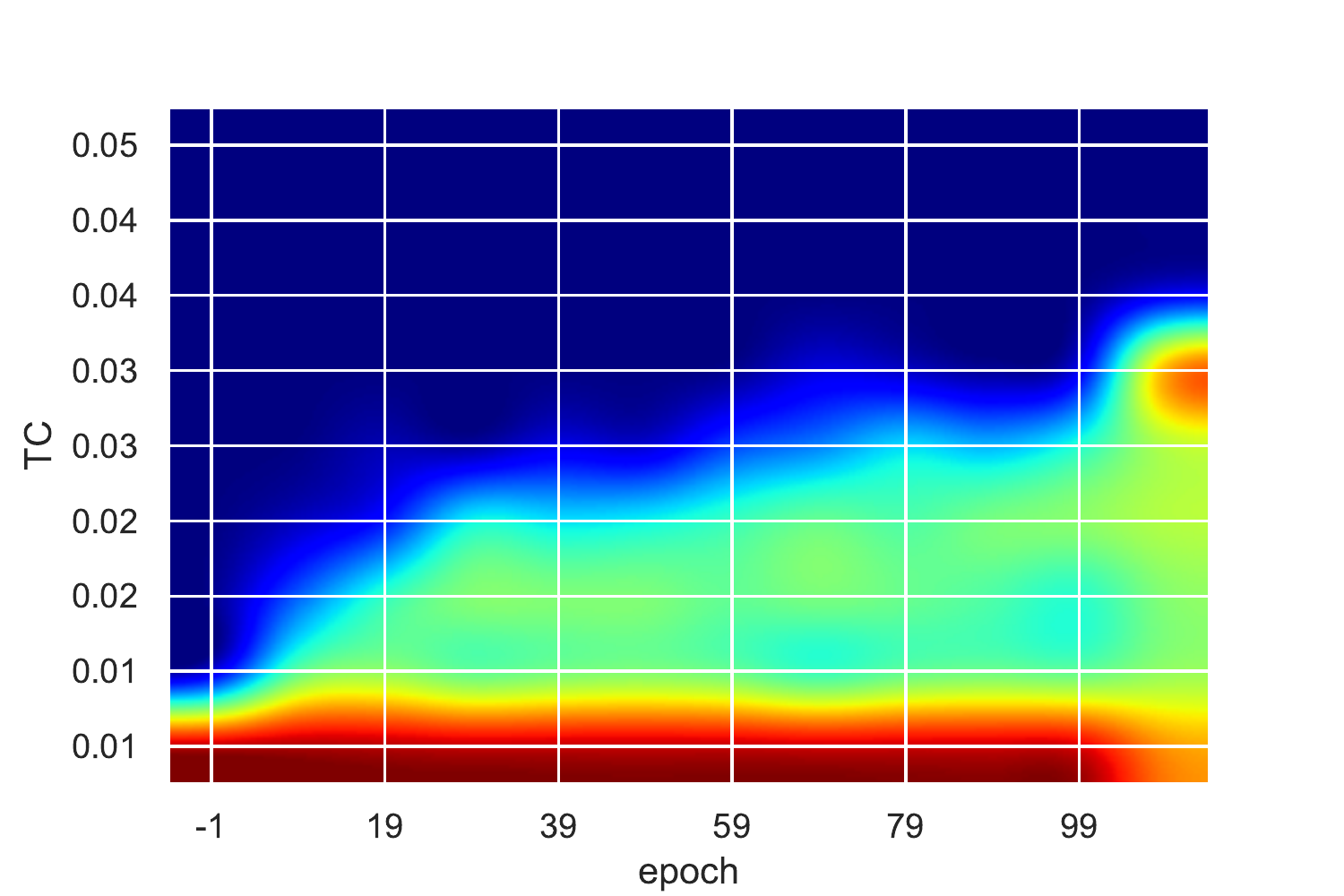}}
\caption{The histogram of HSIC in (a) averaged
from the 1st to the 7th CONV layers; and (b) averaged from
the 8th to 13th CONV layer. Feature maps reach high dependence
with less than 20 epochs of training in lower layers,
but need more than 100 epochs in upper layers.}
\label{fig:HSIC_CNN_100}
\end{figure}

{\color{blue}
\section{The Deep Deterministic Information Bottleneck}
We finally present our preliminary results on parameterizing Tishby \emph{et al.} information bottleneck (IB) principle~\cite{tishby99information} with a neural network. We term our methodology Deep Deterministic Information Bottleneck (DIB), as it avoids variational inference and distribution assumption. We show that deep neural networks trained with DIB outperform the variational objective counterpart and those that are trained with other forms of regularization, in terms of generalization performance.

The IB is an information-theoretic framework for learning. It considers extracting information about a target signal $Y$ through a correlated observable $X$. The extracted information is quantified by a variable $T=T(X)$, which is (a possibly randomized) function of $X$, thus forming the Markov chain $Y \leftrightarrow X \leftrightarrow T$. Suppose we know the joint distribution $p(X,Y)$, the objective is to learn a representation $T$ that maximizes its predictive power to $Y$ subjects to some constraints on the amount of information that it carries about $X$:
\begin{equation}
    \mathcal{L}_{IB}=I(Y;T) - \beta I(X;T),
\end{equation}
where $I(\cdot;\cdot)$ denotes the mutual information. $\beta$ is a Lagrange multiplier that controls the trade-off between the \textbf{sufficiency} (the performance on the task, as quantified by $I(Y;T)$) and the \textbf{minimality} (the complexity of the representation, as measured by $I(X;T)$). In this sense, the IB principle also provides a natural approximation of \emph{minimal sufficient statistic}.

The IB objective contains two mutual information terms: $I(X;T)$ and $I(Y;T)$. When parameterizing IB objective with a DNN, $T$ refers to the latent representation of one hidden layer.
The maximization of $I(Y;T)$ can be replaced by the minimization of a classic cross-entropy loss~\cite{amjad2019learning}. The same trick has also been used in nonlinear information bottleneck (NIB)~\cite{kolchinsky2019nonlinear} and variational information bottleneck (VIB)~\cite{alemi2016deep}.
In this sense, our objective can be interpreted as a cross-entropy loss regularized by a weighted differentiable mutual information term $I(X;T)$. One can simply estimate $I(X;T)$ (in a mini-batch) with our proposed dependence measure (i.e., Eq.~(\ref{eq:MI_unnormalize}) or Eq.~(\ref{eq:MI_max})) in this work.

As a preliminary experiment, we first evaluate the performance of DIB objective on the standard MNIST digit recognition task. We randomly select $10k$ images from the training set as the validation set for hyper-parameter tuning.
For a fair comparison, we use the same architecture as has been adopted in~\cite{alemi2016deep}, namely a MLP with fully connected layers of the form $784-1024-1024-256-10$, and ReLu activation. The bottleneck layer is the one before the softmax layer, i.e., the hidden layer with $256$ units.
The Adam optimizer is used with an initial learning rate of 1$e$-4 and exponential decay by a factor of $0.97$ every $2$ epochs. All models are trained with $200$ epochs with mini-batch of size $100$. Table~\ref{table:Table1} shows the test error of different methods. DIB performs the best.

\setlength{\tabcolsep}{12pt}
\begin{table}[ht]

\centering
\fontsize{10}{12}\selectfont
\caption{Test error (\%) for permutation-invariant MNIST}
 \begin{threeparttable}[t]
\begin{tabular}[t]{lc}
\hline
\bf{Model}&\bf{Test} (\%)\\
\hline
Baseline&1.38\\
Dropout &1.28\\
Label Smoothing~\cite{pereyra2017regularizing}&1.24\\
Confidence Penalty~\cite{pereyra2017regularizing}&1.23\\
VIB~\cite{alemi2016deep}&1.17\tnote{1}\\
\midrule
{\bf DIB} ($\beta$=1$e$-6)&\bf{1.13}\\
\hline
\end{tabular}
 \begin{tablenotes}
     \item[1] \footnotesize{Result obtained on our test environment with authors' code \url{https://github.com/alexalemi/vib_demo}.}
   \end{tablenotes}
    \end{threeparttable}%
\label{table:Table1}
\end{table}%

In our second experiment, we use VGG16~\cite{simonyan2014very} as the baseline network and compare the performance of VGG16 trained by DIB objective and other regularizations. Again, we view the last fully connected layer before the softmax layer as the bottleneck layer. All models are trained
with $400$ epochs, a batch-size of $100$, and an initial learning rate $0.1$. The learning rate was reduced by a factor of $10$ for every $100$ epochs. We use SGD optimizer with weight decay
5$e$-4. We explored $\beta$ ranging from 1$e$-4 to $1$, and found that $0.01$ works the best. Test error rates with different methods are shown in Table~\ref{table:Table2}.
As can be seen, VGG16 trained with our DIB outperforms other regularizations and also the baseline ResNet50. We also observed, surprisingly, that VIB does not provide performance gain in this example, even though we use the authors' recommended value of $\beta$ ($0.01$).

\begin{table}[]

\centering
\fontsize{10}{12}\selectfont
\caption{Test error (\%) on CIFAR-10}
\begin{tabular}[t]{lc}
\hline
\bf{Model}&\bf{Test}(\%)\\
\hline
VGG16&7.36\\
ResNet18&6.98\\
ResNet50&6.36\\
\midrule
VGG16+Confidence Penalty
&5.75\\
VGG16+Label smoothing&5.78\\
VGG16+VIB&9.31\\
\midrule
{\bf VGG16+DIB} ($\beta$=1$e$-2)&\bf{5.66}\\
\hline
\end{tabular}
\label{table:Table2}
\end{table}%

}

\end{document}